\documentclass[11pt]{article}

\usepackage[utf8]{inputenc} 
\usepackage[T1]{fontenc}    
\usepackage{hyperref}       
\usepackage{url}            
\usepackage{booktabs}       
\usepackage{amsfonts}       
\usepackage{nicefrac}       
\usepackage{microtype}      
\usepackage{tikz}
\usetikzlibrary{decorations.pathreplacing,calc}
\newcommand{\tikzmark}[1]{\tikz[overlay,remember picture] \node (#1) {};}

\newcommand*{\AddNote}[4]{%
    \begin{tikzpicture}[overlay, remember picture]
        \draw [decoration={brace,amplitude=0.5em},decorate,ultra thick,blue]
            ($(#3)!(#1.north)!($(#3)-(0,1)$)$) --  
            ($(#3)!(#2.south)!($(#3)-(0,1)$)$)
                node [align=center, text width=2.3cm, pos=0.5, anchor=west] {#4};
    \end{tikzpicture}
}%
\usepackage{algpseudocode}

\usepackage{algorithm}
\usepackage{eqparbox}
\renewcommand{\algorithmiccomment}[1]{\hfill\eqparbox{COMMENT}{\%~#1}}
\usepackage{verbatim}
\usepackage{amsthm}
\usepackage{amsfonts}
\usepackage[nonamelimits]{amsmath}
\usepackage{amssymb}
\usepackage{sansmath}
\usepackage{color}
\usepackage{xcolor}
\usepackage{bm}
\usepackage{bbm}
\usepackage{cite}
\usepackage{fullpage}
\usepackage[bottom]{footmisc}
\newtheorem{theorem}{Theorem}

\newtheorem{proposition}[theorem]{Proposition}

\newtheorem{corollary}[theorem]{Corollary}
\newtheorem{claim}[theorem]{Claim}

\newtheorem{fact}[theorem]{Fact}
\usepackage{subfig}
\usepackage{graphicx}
\usepackage{listings}
\usepackage{enumitem}
\newtheoremstyle{named}{}{}{\itshape}{}{\bfseries}{.}{.5em}{\thmnote{#3}}
\theoremstyle{named}
\newtheorem*{namedtheorem}{Theorem}
\usepackage{tikz}
\usetikzlibrary{fit,calc}

\DeclareSymbolFont{extraup}{U}{zavm}{m}{n}
\DeclareMathSymbol{\varheart}{\mathalpha}{extraup}{86}
\DeclareMathSymbol{\vardiamond}{\mathalpha}{extraup}{87}

\newcommand{\R}{\mathbb{R}}
\newcommand{\RR}{\mathbb{R}}
\newcommand{\PP}{\mathbb{P}}
\newcommand{\EE}{\mathbb{E}}
\newcommand{\E}{\mathbb{E}}

\newcommand{\OO}{\mathcal{O}}

\def\varV{\bm{V}}
\def\varX{\bm{X}}
\def\varXbar{\bm{\overline{X}}}
\def\varY{\bm{Y}}
\def\varZ{\bm{Z}}
\def\varW{\bm{W}}
\def\varx{\bm{x}}
\def\vary{\bm{y}}

\def\varu{\bm{u}}
\def\uX{\underline{X}}
\def\uY{\underline{Y}}
\def\uZ{\underline{Z}}

\providecommand\abs[1]{\lvert#1\rvert}

\DeclareMathOperator\poly{\mathrm{poly}}
\DeclareMathOperator{\pr}{\mathrm{Pr}}

\newcommand{\blue}[1]{\textcolor{blue}{#1}}

\def\eps{\varepsilon}

\def\cond{\mid}

\def\varV{\bm{V}}
\def\varX{\bm{X}}
\def\varXbar{\bm{\overline{X}}}
\def\varY{\bm{Y}}
\def\varZ{\bm{Z}}
\def\varW{\bm{W}}

\providecommand\norm[1]{\|#1\|}

\makeatletter
\newcommand*{\inlineequation}[2][]{%
  \begingroup
    \refstepcounter{equation}%
    \ifx\\#1\\%
    \else
      \label{#1}%
    \fi
    \relpenalty=10000 %
    \binoppenalty=10000 %
    \ensuremath{%
      #2%
    }%
    ~\@eqnnum
  \endgroup
}
\makeatother

\title{Learning and Testing Variable Partitions}
\author{Andrej Bogdanov\thanks{{\tt andrejb@cse.cuhk.edu.hk}. Department of Computer Science and Engineering and Institute for Theoretical Computer Science and Communications, The Chinese University of Hong Kong.  Work funded by Hong Kong RGC GRF grant CUHK14209417.} \and Baoxiang Wang\thanks{{\tt bxwang@cse.cuhk.edu.hk}. Department of Computer Science and Engineering, The Chinese University of Hong Kong. Work funded by Hong Kong RGC GRF grant CUHK14209417. Authors are listed in alphabetical order.}}

\date{}
\date{\vspace{-1ex}}

\begin{document}

\maketitle

\begin{abstract}
Let $F$ be a multivariate function from a product set $\Sigma^n$ to an Abelian group $G$.  A $k$-partition of $F$ with cost $\delta$ is a partition of the set of variables $\varV$ into $k$ non-empty subsets $(\varX_1, \dots, \varX_k)$ such that $F(\varV)$ is $\delta$-close to $F_1(\varX_1)+\dots+F_k(\varX_k)$ for some $F_1, \dots, F_k$ with respect to a given error metric.  We study algorithms for agnostically learning $k$ partitions and testing $k$-partitionability over various groups and error metrics given query access to $F$.  In particular we show that
\begin{enumerate}
\item Given a function that has a $k$-partition of cost $\delta$, a partition of cost $\OO(k n^2)(\delta + \eps)$ can be learned in time $\tilde{\OO}(n^2 \poly (1/\eps))$ for any $\eps > 0$.  In contrast, for $k = 2$ and $n = 3$ learning a partition of cost $\delta + \eps$ is NP-hard.

\item When $F$ is real-valued and the error metric is the 2-norm, a 2-partition of cost $\sqrt{\delta^2 + \eps}$ can be learned in time $\tilde{\OO}(n^5/\eps^2)$.

\item When $F$ is $\mathbb{Z}_q$-valued and the error metric is Hamming weight, $k$-partitionability is testable with one-sided error and $\OO(kn^3/\eps)$ non-adaptive queries.  We also show that even two-sided testers require $\Omega(n)$ queries when $k = 2$.
\end{enumerate}
This work was motivated by reinforcement learning control tasks in which the set of control variables can be partitioned. The partitioning reduces the task into multiple lower-dimensional ones that are relatively easier to learn. Our second algorithm empirically increases the scores attained over previous heuristic partitioning methods applied in this context.
\end{abstract}

\section{Introduction}
\label{sec:intro}

Divide-and-conquer methods rely on the ability to identify independent sub-instances of a given instance, such as connected components of graphs and hypergraphs.  When these are not available one looks for partitions into loosely related parts like small or sparse cuts.   These classic problems and their variants remain at the forefront of algorithmic research~\cite{karger-levine,kawarabayashi-thorup,chekuri-li,manurangsi,chandrasekaran-xu-yu,rubinstein-schramm-weinberg}.

We study the related problem of function decomposition:   Given a multivariate function $F(\varV)$ over $n$ variables $\varV = \{\varx_1, \dots, \varx_n\}$, we seek to partition the variables into $k$ groups $\varX_1, \dots, \varX_k$ so that $F$ decomposes into a sum $F_1(\varX_1) + \cdots + F_k(\varX_k)$.  In case an exact decomposition of this type is unavailable, we seek an approximate one under a suitable error metric.  This algebraic partitioning question can be sensibly asked for any Abelian group.  While some of our results are quite general, two particular cases of interest are addition over $\mathbb{Z}_2$ with respect to the Hamming metric and addition over reals with respect to the $2$-norm.

As a multivariate function is an exponentially large object, it is sensible to model the input $F$ to the partitioning problem as an oracle and allow query access to it.  This departs from the common setup in (hyper)graph partitioning problems, where an explicit representation of the input is assumed to be available.  While variable partitioning of real-valued functions under the 2-norm turns out to be closely related to hypergraph partitioning, the difference in input access models renders certain techniques developed for the latter (e.g., random contractions) inapplicable to our setting.

Our work is motivated by learning control variables in high-dimensional reinforcement learning control~\cite{sutton2018reinforcement,mnih2016asynchronous,sutton2000policy}.
If the \textit{advantage function} of the control variables can be partitioned into multiple lower-dimensional subsets, then these subsets of variables can be learned independently with a relatively easier Monte-Carlo sampling.
This advantage function involves the estimates of a dynamic system, which is complex enough to not have an explicit representation available. The function is thus treated as an oracle as is in our access model.
Sometimes it is natural to assume that the function should be almost decomposable; for example, if we seek to control two robots jointly performing a task, the variables controlling the respective robots are almost independent. (The robots may be collaborating so the decomposition might not be perfect.)  
In general, the dependencies are not known in advance but need to be learned from observed behavior. Some heuristic methods have been applied to control variable partitioning~\cite{wu2018variance,li2018policy} but not rigorously analyzed.

\paragraph{Our contributions}  Our main results are algorithmic:  We show that variable partitions can be learned agnostically.

Let $F(\varV)$ be a function from some product set to an Abelian group $G$.  A direct sum decomposition of $F$ is a partition $(\varX_1, \dots, \varX_k)$ of the set of variables $\varV$ such that $F(\varV)$ is $F_1(\varX_1) + \cdots + F_k(\varX_k)$ for some functions $F_1, \dots, F_k$. When the decomposition is imperfect, the decomposition error is measured by
\begin{align}
\label{eqn:delta}
& \delta(\bm{X}_1, \dots, \bm{X}_k) = \min_{F_1, \dots, F_k} \norm{F(X_1, \dots, X_k) - F_1(X_1) - \dots - F_k(X_k)},
\end{align}
where $\norm{\cdot}\colon G \to \R^+$ is a partial norm.  The definition is given in Section~\ref{sec:formulation}; the main examples of interest are $G = \mathbb{Z}_q$ under the Hamming metric
$\norm{F} = \pr[F(V) = 0]$ and $G = \R$ under the $p$-norm $\norm{F}_p = \EE\bigl[\abs{F(V)}^p\bigr]^{1/p}$ for any $p \geq 1$ under some product measure.  We seek an approximation of the best-possible partition, which minimizes the objective
\begin{equation}
\label{eqn:delta2}
\delta_2(F) = \min_{\varX} \delta(\varX, \overline{\varX}),    
\end{equation}
for bipartition and
\begin{equation}
\label{eqn:deltak}
\delta_k(F) = \min_{\varX_1, \dots, \varX_k} \delta(\varX_1, \dots, \varX_k).
\end{equation}
for $k$-partition.  (For $p$-norms over $\R$ we use the notations $\norm{\cdot}_{\R,p}$, $\delta_{\R,2}(F)$, and $\delta_{\R,k}(F)$.)

\begin{theorem}
\label{thm:first}
Let $\norm{\cdot}$ be either 1) $\norm{\cdot}_{\mathbb{R}, p}$ assuming $\norm{F}_{\mathbb{R}, 2p} = \OO(1)$, or 2) the Hamming metric over $\mathbb{Z}_q$.  There is an algorithm that given parameters $n$, $k$, $\eps$, $\gamma$, and oracle access to $F\colon \Sigma^n \to G$ outputs a $k$-partition $\mathcal{P}$ such that $\delta(\mathcal{P}) \leq \OO(kn^2) (\delta_k(F) + \eps)$ with probability at least $1 - \gamma$.  The algorithm makes $\OO(K^p n^2 \log(n/\gamma)/\eps^{2p})$ queries to $F$ and runs in time linear in the number of queries, for an absolute constant $K$.
\end{theorem}

This algorithm is closely related to the heuristic ones used in the aforementioned empirical studies.   However, it only guarantees optimality up to an $\OO(kn^2)$ approximation factor.  While we do not know if an approximation factor of this magnitude is inevitable, in Proposition~\ref{prop:nphard1} we show that obtaining a solution with additive error is NP-hard.  The proofs are given in Section~\ref{sec:first}.

In contrast, our second algorithm obtains an additive error for bipartitions of real-valued functions under the 2-norm:

\begin{theorem}
\label{thm:l2est}
Let $F\colon \Sigma^n \to \R$ be a function with $\norm{F}_{\R, 4} \leq 1$.  There is an algorithm that given inputs $n$, $\eps$, $\gamma$, and oracle access to $F$, runs in time $\OO(n^5\log(n/\gamma)/\eps^2)$ and outputs a bipartition $(\varX, \overline{\varX})$ such that $\delta_{\mathbb{R}, 2}(\varX, \overline{\varX})^2 \leq \delta_{\mathbb{R}, 2}(F)^2 + \eps$ with probability at least $1 - \gamma$.
\end{theorem}

More generally, we show that it is possible to output a $\sqrt{2 - 2/k}$-approximate $k$-partition in time $\poly(n^k, k, 1/\eps)$ (Corollary~\ref{cor:sarvaz}).  For unbounded $k$ finding a good approximation is ETH hard (Corollary~\ref{cor:hard2}).

Theorem~\ref{thm:l2est} and Corollary~\ref{cor:hard2} are based on an equivalence between variable partitioning under the 2-norm and hypergraph partitioning given in Proposition~\ref{prop:efronstein}.  The results are described and proved in Section~\ref{sec:second}.

As a consequence of Theorem~\ref{thm:first}, the property of being close to a $k$-partition is testable with $\tilde{\OO}(k^{2p} n^{4p + 2}/\eps^{2p})$ queries.  The query complexity of the tester can be somewhat improved:

\begin{theorem}
\label{thm:tester}
$k$-partitionability is testable with one-sided error and $\OO(kn^3/\eps)$ non-adaptive queries with respect to Hamming weight over $\mathbb{Z}_q$, and with $\OO(k^{2p}n^3/\eps^{2p})$ non-adaptive queries with respect to the $p$-norm over $\mathbb{R}$ assuming $\norm{F}_{2p} \leq 1$.
\end{theorem}

In Section~\ref{sec:testing} we prove Theorem~\ref{thm:tester} and show that $\Omega(n - k)$ queries are necessary even for two-sided error testers.

\begin{table}[t!]
\label{tbl:notation}
\centering
\begin{tabular}{llll}
\toprule
Notation & Meaning & Notation & Meaning\\
\midrule
$\varx, \vary \in \varV$ & variables &
$\delta_k(F)$  & optimal $k$-partition error \\
$\varX, \varY, \varXbar \subseteq \varV$ & sets of variables & $D_F(\varX,\varY)$ & dependence score\\
$x, y, X, Y$ & (random) assignments & $\norm{\cdot}, \norm{\cdot}_{\RR,p}$ & partial norm and $p$-norm \\
\bottomrule
\end{tabular}
\end{table}

\paragraph{Ideas and techniques}  Our Theorem~\ref{thm:first} is inspired by algebraic property testing techniques.  The starting point is the dual characterization of partitionability into sets $(\varX, \overline{\varX})$ by the constraints $D_F(\varX, \varY) = 0$, where $D_F = F(X, Y) - F(X^\prime, Y) - F(X, Y^\prime) + F(X^\prime, Y^\prime)$, for all assignments $X, X'$ to $\varX$ and $Y, Y'$ to $\varY$.  David et al.~\cite{david-etal} apply this relation to random inputs towards testing whether a $\mathbb{Z}_2$-valued function $F$ tensors decomposes into a direct sum.  The acceptance probability of this test approximates the best decomposition to within a factor of 4 (Proposition~\ref{prop:known}).

Our partitioning algorithm estimates the \textit{dependence score} $\norm{D_F(\varx, \vary)}$ on every pair of variables $\varx, \vary$ (keeping the rest fixed) to decide whether they should be partitioned or not.  Here, $\norm{D_F}$ is the probability that the test $D_F=0$ fails for discrete groups like $\mathbb{Z}_2$.  In general, it can represent any error metric satisfying the axioms in Section~\ref{sec:formulation}.  The proof of Theorem~\ref{thm:first} amounts to showing that a collection of single variable partitions $(\varx, \vary) \in \mathcal{P}$ for which the local scores $\norm{D_F(\varx, \vary)}$ are small can be glued together into a single $k$-partition $\mathcal{P}$ with a small global score.  

When $F$ is real-valued and error is measured under the 2-norm, variable partitioning has a natural geometric interpretation. Functions that depend on different coordinates are orthogonal modulo their constant term, so the optimal decomposition with respect to a fixed partition $(\varX_1, \dots, \varX_k)$ is given by the projection of $F$ onto the respective subspaces of functions.  This yields an equality between the distance and the dependence score~\eqref{eq:dfr2} for bipartitions and a generalization to $k$-partitions (Proposition~\ref{prop:delta2}).  Variable partitioning for functions is then equivalent to hypergraph partitioning of their orthogonal decompositions (Proposition~\ref{prop:efronstein}), with the cost of cut $(\varX, \overline{\varX})$ given by $\tfrac14\norm{D_F(\varx, \vary)}^2$.  

This connection suggests the application of hypergraph partitioning algorithms that can be implemented with access to an {\em approximate} cut oracle\footnote{Several state-of-the-art algorithms for cuts in graphs and hypergraphs rely on random contractions~\cite{karger, karger-stein, chandrasekaran-xu-yu}.  In particular, Rubinstein et al.~\cite{rubinstein-schramm-weinberg} showed that $\tilde{\OO}(n)$ queries to an {\em exact} cut oracle and similar running time are sufficient to find the minimum cut.  We do not know if comparable efficiency can be obtained with an approximate oracle.}, leading to Theorem~\ref{thm:l2est}.  On the negative side it reveals that approximately optimal partitions into a large number of components are hard to find (Corollary~\ref{cor:hard2}).

\paragraph{Application to reinforcement learning control}
We plug our partitioning algorithm back to reinforcement learning control. In this setting, the oracle is real-valued and as we adapt the 2-norm we use the submodularity cut algorithm described in Theorem \ref{thm:l2est}.

We compare empirically with three previous approaches: The baseline that does not involve partitioning \cite{mnih2016asynchronous,williams1992simple}; the baseline that trivially partitions $n$ variables into $n$ subsets \cite{wu2018variance,pytorchrl}; the work that partitions the variables heuristically \cite{li2018policy}. The way \cite{li2018policy} partitions the variables is to calculate the discrete estimate of the Hessian of the oracle. Then they remove from Hessian the elements with lowest absolute values, until it forms at least $k$ connected components if the Hessian matrix is treated as the adjacency matrix. 

The scores we attained on the tasks in the physics simulator are improved over these approaches, which is demonstrated in Section~\ref{sec:apprl}.

\paragraph{Relation to other learning and testing problems}  A $j$-junta is a function that depends on at most $j$ of its $n$ variables.  The problems of learning and testing juntas have been extensively studied \cite{mossel2003learning, fischer2004testing, chockler-gutfreund, blais, saglam, chen-etal, bshouty}.  While a $j$-junta is always $(n - j + 1)$-partitionable, the two problems are technically incomparable.  Moreover, juntas are usually studied in the regime where the junta size $j$ is significantly smaller than the number of variables $n$ and are therefore partitionable into many (mostly trivial) components.  In this work we are mostly interested in partitions into two or a small number of components.  Nevertheless, this connection between juntas and partitionable functions is used to prove the testing lower bound in Section~\ref{sec:testing}.

Dinur and Golubev~\cite{dinur-golubev} showed that the existence of decomposition with respect to a fixed $k$-partition (given as input) is testable with four queries and soundness error $\Omega(\delta)$.  The case $k = 2$ was already analyzed by David et al.~\cite{david-etal} (see Section~\ref{sec:test}).

\section{Some additional definitions}
\label{sec:formulation}

Let $F(\varV)$ be a function from some product set to an Abelian group $G$.  In general we will assume that the variables $\varV$ take values in some set $\Sigma$ endowed with a product measure which is efficiently sampleable.  The quality of the partition $(\bm{X}_1, \dots, \bm{X}_k)$ of $\bm{V}$ is measured by $\delta(\bm{X}_1, \dots, \bm{X}_k)$ given in~\eqref{eqn:delta}, where $\norm{\cdot}\colon G \to \R^+$ can be any functional satisfying the following three axioms:
\begin{enumerate}
\item $\norm{0} = 0$;
\item $\norm{F_1 + F_2} \leq \norm{F_1} + \norm{F_2}$;
\item $\E[\norm{F(X, \cdot)} \cond X] \leq \norm{F}$ for any set of variables $X$ of $F$.
\end{enumerate}
Our goal is to approximately optimize $\delta_2(F)$ in \eqref{eqn:delta2} and $\delta_k(F)$ in \eqref{eqn:deltak}.

Our algorithms are based on the following dependence estimator inspired by the rank-1 test of~\cite{david-etal}.  Let $\varX$ and $\varY$ be two disjoint sets of variables.  The dependence estimator $D_F(\varX, \varY)$ is the random variable
\begin{align*}
\label{eqn:dependencescore}
D_F & = F(X, Y, Z) + F(X', Y', Z) - F(X', Y, Z) - F(X, Y', Z), 
\end{align*}
where $X, X'$ are independent samples of the $\varX$ variable, $Y, Y'$ are independent samples of the $\varY$ variable, and $Z$ is a random sample of the remaining variables.  If $F$ decomposes into a direct sum that partitions the $\varX$ and $\varY$ variables then $D_F$ equals zero. Conversely, $\norm{D_F}$ measures the quality of the approximation.  

In the analysis it will be convenient to use the notation $F \approx_\delta G$ for $\norm{F - G}_p \leq \delta$.  The following two facts are immediate consequences of axioms 2 and 3:
\begin{description}
\item {\bf Triangle inequality:} If $F \approx_\delta G$ and $G \approx_{\delta'} H$ then $F \approx_{\delta + \delta'} H$.
\item {\bf Fixing:} If $F(X, Z) \approx_\delta G(X, Z)$ then $F(\uX, Z) \approx_\delta G(\uX, Z)$ for some fixed value $\uX$.
\end{description}


\section{Estimating the quality of a partition}
\label{sec:test}

In this section we show that $\norm{D_F(\varX, \varY)}$ is an approximate estimator for the quality $\delta(\varX, \varY)$ of a decomposition, namely
\begin{equation}
\label{eq:factor4}
\delta(\varX, \varY) \leq \norm{D_F(\varX, \varY)} \leq 4 \cdot \delta(\varX, \varY).
\end{equation}

The proof is given in Claims~\ref{claim:compl} and~\ref{claim:base} below.  As $\norm{D_F(\bm{X}, \bm{Y})}$ can be estimated efficiently from oracle access to $F$ (Claim~\ref{claim:estD}), we obtain an algorithm for estimating the quality of a partition to within a factor of 4 in general, and exactly for the 2-norm over $\R$.

\begin{proposition}
\label{prop:known}
Let $\norm{\cdot}$ be either 1) $\norm{\cdot}_{\mathbb{R}, p}$ assuming $\norm{F}_{\mathbb{R}, 2p} = \OO(1)$, or 2) the Hamming metric over $\mathbb{Z}_q$.
There is an algorithm that given a bipartition $\varX, \varY$ of the variables and parameters $\eps, \gamma > 0$, outputs a value $\hat{\delta}$ such that 
\[
\delta(\varX, \varY) \leq \hat{\delta} \leq 4 \cdot \delta(\varX, \varY) + \eps,
\]
with probability at least $1 - \gamma$ from $K^p{\log(1/\gamma)}/{\epsilon^{2p}}$ queries to $F$ in time linear in the number of queries, for an absolute constant $K$.
\end{proposition}

The value of $\delta(\varX, \varY)$ is known to be NP-hard to calculate exactly over $\mathbb{Z}_2$ under the Hamming metric given explicit access to the truth-table of $F$~\cite{roth-viswanathan}.  Therefore some approximation factor is unavoidable for algorithms running in time polynomial in $n$ and $1/\eps$ unless BPP is in NP.  On the positive side Karpinski and Schudy~\cite{karpinski-schudy} give a fully polynomial-time randomized approximation scheme for this special case.  Their algorithm requires at least linear time but it is plausible that a sublinear-time variant can be obtained.  However, it appears unrelated to the dependence score $D_F$ which plays an essential role in the results to follow.

The analysis of $D_F$ applies to any pair of disjoint subsets $\varX$, $\varY$ that do not necessarily partition all the variables.  In this more general setting distance is measured by the formula
\begin{equation}
\label{eq:anypair}
\delta(\varX, \varY) = \min_{A, B} \norm{F(X, Y, Z) - A(X, Z) - B(Y, Z)}.
\end{equation}

\begin{claim}[Completeness of $D_F$]  
\label{claim:compl}
For all disjoint $\bm{X}$, $\bm{Y}$, $\norm{D_F(\bm{X}, \bm{Y})} \leq 4\cdot\delta(\bm{X}, \bm{Y})$. 
\end{claim}
\begin{proof}
By definition of $\delta(\bm{X}, \bm{Y})$ there exists a decomposition of the form 
\[ F(X, Y, Z) = A(X, Z) + B(Y, Z) + D(X, Y, Z), \]
where $\norm{D(X, Y, Z)} = \delta(\bm{X}, \bm{Y})$. In the expansion of $D_F$ all the $A$ and $B$ terms cancel out, leaving
\begin{align*}
\norm{D_F(\bm{X}, \bm{Y})} & = \norm{D(X, Y, Z) + D(X', Y', Z) - D(X, Y', Z) - D(X', Y, Z)} \\
& \leq \norm{D(X, Y, Z)} + \norm{D(X', Y', Z)} + \norm{D(X, Y', Z)} + \norm{D(X', Y, Z)} \\
& = 4\delta(\bm{X}, \bm{Y}). \tag*{\qedhere}
\end{align*}
\end{proof}

Soundness for Boolean functions under the uniform measure was proved by David et al.~\cite{david-etal}.  We reproduce their proof under a more general setting.

\begin{claim}[Soundness of $D_F$]
\label{claim:base}
For all disjoint $\bm{X}$, $\bm{Y}$, $\delta(\varX, \varY) \leq \norm{D_F(\bm{X}, \bm{Y})}$.
\end{claim}

\begin{proof}
Let $\eps = \norm{D_F(\bm{X}, \bm{Y})}$.  Then
\[ F(X, Y, Z) \approx_\eps F(X, Y', Z) - F(X', Y, Z) - F(X', Y', Z). \]
We can fix values $\uX'$ and $\uY'$ for which
\[ F(X, Y, Z) \approx_\eps F(\uX', \uY', Z) - F(X, \uY', Z) - F(\uX', Y, Z) = A(X, Z) + B(Y, Z), \]
where $A(X, Z) = F(\uX', \uY', Z) - F(X, \uY', Z)$ and $B(Y, Z) = F(\uX', Y, Z)$.
\end{proof}

Proposition~\ref{prop:known} now follows from inequality~\eqref{eq:factor4} and the following claim, which states that $\norm{D_F}$ can be estimated by sampling in the cases of interest.  See Appendix \ref{appendix:proof1} for the proof.

\begin{claim}
\label{claim:estD}
Assuming $\norm{F}_{\mathbb{R}, 2p} \leq 1$, the value $\norm{F}_{\mathbb{R}, p}$ can be estimated within $\eps$ from $K^p{\log(1/\gamma)}/{\epsilon^{2p}}$ (random) queries to $F$ in linear time with probability $1 - \gamma$ for some absolute constant $K$.
\end{claim}

\subsection{Exact partitioning under the 2-norm}

Since computing the optimal partition is in general NP-complete, we do not expect to replace the inequalities in~\eqref{eq:factor4} with an equality. However, in the special case of real-valued functions with $2$-norm,
the estimate becomes exact:
\begin{equation}
\label{eq:dfr2}
\norm{D_F(\varX, \varY)}_{\mathbb{R}, 2} = 2\cdot\delta_{\mathbb{R}, 2}(\varX, \varY).
\end{equation}

This equality is a consequence of the following characterization of $\delta_{\mathbb{R}, 2}$, which applies more generally to $k$-partitions:

\begin{proposition}
\label{prop:delta2}
Assuming $\E[F] = 0$, the $k$-partition $F_i(X_i) = \E[F | X_i]$ achieves the minimum for $\delta_{\mathbb{R}, 2}(\varX_1, \dots, \varX_k)$. 
\end{proposition}

In particular it follows that $\delta_{\mathbb{R}, 2}$ takes the value
\begin{equation}
\label{eq:pF}
\delta_{\mathbb{R}, 2}(\varX_1, \dots, \varX_k) = \E[(\overline{F} - \E[\overline{F} | X_1] - \dots - \E[\overline{F} | X_k])^2],
\end{equation}
where $\overline{F} = F - \E[F]$.  To derive identity~\eqref{eq:dfr2} it remains to verify that when $k = 2$, the right-hand side of~\eqref{eq:pF} is a quarter of $\norm{D_F}^2$:

\begin{fact}
\label{claim:dfid}
$\norm{D_F(\varX, \varY)}_{\mathbb{R}, 2}^2 = 4\cdot\E[(\overline{F} - \E[\overline{F} | X] - \E[\overline{F} | Y])^2]$.
\end{fact}
Armed with this fact we prove the proposition.

\begin{proof}[Proof of Proposition~\ref{prop:delta2}]
First assume $F(\varX, \varY)$ is bivariate. Let $A(\varX)$ be any function.  The inequality $\E[(\E[F | X] - A(X))^2] \geq 0$ can be rewritten as
\begin{equation}
\label{eqn:one}
\E[(F - \E[F | X])^2] \leq \E[(F - A(X))^2],
\end{equation}
stating that the orthogonal projection of $F$ onto the subspace of functions that depend only on $\varX$ in $2$-norm is $\E[F | X]$.

Now let $F(\varX_1, \dots, \varX_k)$ be $k$-variate.  Assume $\E[F] = 0$ and $\E[F_i(X_i)] = 0$ for all $i$.  Then $\E[F_i(X_i) | X_j] = 0$ for all $i \neq j$.  Applying inequality~\eqref{eqn:one} for $k$ times in succession together with this fact, we obtain
\begin{multline*} 
\E[(F - F_1(X_1) - \cdots - F_{k-1}(X_{k-1}) - F_k(X_k))^2] \\
\begin{aligned}
  &\geq \E[(F - F_1(X_1) - \cdots - F_{k-1}(X_{k-1}) - \E[F - F_1(X_1) - \cdots - F_{k-1}(X_{k-1}) | X_k])^2] \\
  &= \E[(F - F_1(X_1) - \cdots - F_{k-1}(X_{k-1}) - \E[F | X_k])^2] \\
  &\ \ \vdots \\
  &\geq \E[(F - \E[F | X_1] - \cdots - \E[F | X_k])^2]
\end{aligned}
\end{multline*}
as desired.  Finally, by orthogonality the optimal decomposition must satisfy $\sum \E[F_i(X_i)] = 0$ so the assumption 
$\E[F_i(X_i)] = 0$ can be made without loss of generality.
\end{proof}

By orthogonality, equation~\eqref{eq:pF} can also be written in the following forms:
\begin{align}
\delta_{\mathbb{R}, 2}(\varX_1, \dots, \varX_k) 
&= \E\bigl[\overline{F}^2\bigr] - \sum\nolimits_{i=1}^k \E\bigl[\E[\overline{F} | X_i]^2\bigr]  \label{eq:deltaalt} \\
&= \E_X\bigl[\overline{F}(X)^2\bigr] - \sum_{i = 1}^k \E_{X, X'}\bigl[\overline{F}(X_{-i}, X_i) \overline{F}(X'_{-i}, X_i)\bigr], \notag
\end{align}
where $(X_{-i}, X_i)$ is the input whose $i$-th variable takes value $X_i$ and $j$-th variable takes value $X'_j$ for $j \neq i$.  As all these terms can be efficiently estimated, we obtain the following algorithm for estimating the quality of a given $k$-partition:

\begin{proposition}
\label{prop:l2given}
There is an algorithm that given a $k$-partition $\varX_1, \cdots, \varX_k$ of the variables and parameters $\eps, \gamma > 0$, outputs a value $\hat{\delta}$ such that 
\[
\abs{\hat{\delta}^2 - \delta_{\mathbb{R}, 2}(\varX_1, \dots, \varX_k)^2} \leq \eps,
\]
with probability at least $1 - \gamma$ from $\OO(k \log(k/\gamma)/\epsilon^4)$ queries to $F$ in time linear in the number of queries.  
\end{proposition}

\section{Variable partitioning over general groups}
\label{sec:first}

In this section we present our first partitioning algorithm, which is general enough to work on any normed group $G$ assuming it is possible to efficiently estimate the quantity $\norm{D_F(\{\varx\}, \{\vary\})}$. The algorithm outputs an $\OO(kn^2)$ approximation to the optimal partition in time polynomial in $n$, $k$, and $1/\eps$.

The algorithm is based on the pairwise estimates of dependency over sets of single variables. The intuition behind the algorithm is that if the dependency between $\varx$ and $\vary$ is low, then these two variables should be assigned to different partitions. Therefore the algorithm keeps asserting such ``in different partitions'' for the pairs with the lowest dependency estimates, until the $k$-partitioning can be clearly observed from the assertions. It is worth noting that this idea of the algorithm has been used in previous works in reinforcement learning control \cite{wu2018variance,li2018policy} in a heuristic way.  

\begin{algorithm}[!htbp]
\caption{Approximate partitioning via pairwise estimates}
\begin{algorithmic}[1]
\State {\bfseries Input:} number of sets $k$ in partition
\State {\bfseries Output:} partition $\mathcal{P}$
\State For every pair of variables $\varx, \vary\in \varV$, find estimate $\hat{e}(\varx, \vary)$ for $e(\varx, \vary) = \norm{D_F(\{\varx\}, \{\vary\})}$;
\State Create a weighted graph with vertices $\varV$ and weights $\hat{e}(\varx, \vary)$;
\State Order the edges in increasing weight;
\Repeat
\State Remove the edge with the smallest weight;
\Until{The graph has exactly $k$ connected components}
\end{algorithmic}
\label{algo}
\end{algorithm}

\begin{proposition}
\label{prop:first}
Assuming $e(\varx, \vary) \leq \hat{e}(\varx, \vary) \leq e(\varx, \vary) + \eps$ for all $\varx$ and $\vary$, 
\begin{equation}
\label{eq:apxfirst}
\delta(\mathcal{P}) \leq (8k - 10)n^2(4\delta_k(F) + \eps).
\end{equation}
\end{proposition}


If the estimates $\hat{e}(\varx, \vary)$ are obtained by empirical averaging, we obtain Theorem~\ref{thm:first}.


\subsection{Proof of Proposition~\ref{prop:first} and Theorem \ref{thm:first}}
\label{sec:prooffirst}

For a partition $\mathcal{P}$ of the variables, let $\Delta(\mathcal{P}) = \sum \delta(\{\varx\}, \{\vary\})$, where the sum is taken over all pairs that cross the partition.  We will deduce Theorem~\ref{thm:first} from the following bound on $\delta(\mathcal{P})$.

\begin{claim}
\label{claim:kpart}
For every $k$-partition $\mathcal{P}$, $\delta(\mathcal{P}) \leq (16k - 20) \Delta(\mathcal{P})$.
\end{claim}

The following fact is immediate from the definitions of $\delta$. The proof of this claim is delayed to the end of this section.

\begin{fact}
\label{fact:ext}
For any partition $(\bm{U}, \bm{\overline{U}})$ such that $\bm{X} \subseteq \bm{U}$ and $\bm{Y} \subseteq \bm{\overline{U}}$, $\delta(\bm{X}, \bm{Y}) \leq \delta(\bm{U}, \bm{\overline{U}})$.
\end{fact}

Now we prove the Proposition \ref{prop:first} and the Theorem \ref{thm:first}, assuming the correctness of Claim \ref{claim:kpart}.

\begin{proof}[Proof of Proposition \ref{prop:first}]
By Claim~\ref{claim:compl} and Fact~\ref{fact:ext}, all edges $(\varx, \vary)$ in the optimal partition must satisfy $e(\varx, \vary) \leq 4\delta_2(F)$.  By our assumption on the quality of the approximations, 
\begin{equation}
\label{eq:apx}
\hat{e}(\varx, \vary) \leq 4\delta_2(F) + \eps.
\end{equation} 
Since the algorithm removes edges in increasing order of weight, all the edges that cross the output partition $\mathcal{P}$ must also satisfy this inequality.  Then
\begin{align*}
\delta(\mathcal{P}) &\leq (16k - 20) \Delta(\mathcal{P}) &&\text{by Claim~\ref{claim:kpart},} \\
  &\leq (16k - 20) \sum\nolimits_{\text{$\varx, \vary$ cross $\mathcal{P}$}} e(\varx, \vary) &&\text{by Claim~\ref{claim:base},} \\
  &\leq (16k - 20) \sum\nolimits_{\text{$\varx, \vary$ cross $\mathcal{P}$}} \hat{e}(\varx, \vary) 
  &&\\
  &\leq (16k - 20) \sum\nolimits_{\text{$\varx, \vary$ cross $\mathcal{P}$}} 4\delta_2(F) + \eps 
  &&\text{by \eqref{eq:apx},} \\
  &\leq (8k - 10)n^2 \cdot (4 \delta_2(F) + \eps).
\end{align*}
The last inequality holds because there are at most $\binom{n}{2} \leq n^2/2$ pairs of variables crossing the partition.  
\end{proof}

\begin{namedtheorem}[Theorem \ref{thm:first}]
Let $\norm{\cdot}$ be either 1) $\norm{\cdot}_{\mathbb{R}, p}$ assuming $\norm{F}_{\mathbb{R}, 2p} = \OO(1)$, or 2) the Hamming metric over $\mathbb{Z}_q$.  There is an algorithm that given parameters $n$, $k$, $\eps$, $\gamma$, and oracle access to $F\colon \Sigma^n \to G$ outputs a $k$-partition $\mathcal{P}$ such that $\delta(\mathcal{P}) \leq \OO(kn^2) (\delta_k(F) + \eps)$ with probability at least $1 - \gamma$.  The algorithm makes $\OO(K^p n^2 \log(n/\gamma)/\eps^{2p})$ queries to $F$ and runs in time linear in the number of queries, for an absolute constant $K$.
\end{namedtheorem}

\begin{proof}
By Proposition \ref{prop:known} we can estimate an edge $e(\varx,\vary)$ up to $\eps$-error using $\OO(K^p{\log(n/\gamma)}/{\epsilon^{2p}})$ samples with probability at least $1-\gamma/n$, for some absolute constant $K$. Therefore with $n^2$ times the amount of samples, which is $\OO(K^p n^2\log(n/\gamma)/\eps^{2p})$, we can estimate all edges up to $\eps$-error with probability at least $1-\gamma$. Then by Proposition \ref{prop:first} we have $\delta(\mathcal{P}) \leq (8k - 10)n^2(4\delta_k(F) + \eps)$ with probability at least $1-\gamma$.
\end{proof}

We first prove Claim~\ref{claim:kpart} in the case $k = 2$ of bipartitions.  This is Claim~\ref{claim:cuts} below.  We use $\varX\varX'$ to denote the union of the variable sets $\varX$ and $\varX'$.

\begin{claim}
\label{claim:induct}
For disjoint sets of variables $\bm{X}, \bm{X'}, \bm{Y}$, $\delta(\bm{X X'}, \bm{Y}) \leq \delta(\bm{X}, \bm{Y}) + 2\delta(\bm{X}', \bm{Y})$.
\end{claim}

\begin{proof}
Assume that
\begin{align*}
F(X, X', Y) &\approx_\delta A(X, X') + B(X', Y) \qquad\text{and} \\
F(X, X', Y) &\approx_{\delta'} A'(X, X') + B'(X, Y).
\end{align*}
By the triangle inequality,
\[
A(X, X') + B(X', Y) \approx_{\delta + \delta'} A'(X, X') + B'(X, Y). 
\]
Fix $X'(Z) = \uX'(Z)$.  Writing $C(X) = A(X, \uX') - A'(X, \uX')$ and $D(Y') = B(\uX', Y')$ we get that
\[
B'(X, Y) \approx_{\delta + \delta'} C(X) + D(Y).
\]
By the triangle inequality (with the second equation), we get that
\[
F(X, X', Y) \approx_{\delta + 2\delta'} A'(X, X') + C(X) + D(Y). 
\hfill\tag*{\qedhere} \]
\end{proof}

\begin{claim}
\label{claim:cuts}
For every bipartition $\varX, \bm{\overline{X}}$ of the variables, $\delta(\varX, \bm{\overline{X}}) \leq 4 \cdot \Delta(\bm{X}, \bm{\overline{X}})$. 
\end{claim}
\begin{proof}
By Claim~\ref{claim:induct},
\[ \delta(\varX'\{\bm{x}\}, \{\bm{y}\}) \leq \delta(\varX', \{\vary\}) + 2\delta(\{\varx\}, \{\vary\}) \]
for all $\varX' \subseteq \varX \setminus \{x\}$ and $\vary$.  Applying this inequality iteratively we conclude that $\delta(\varX, \{\vary\}) \leq 2\sum_{\varx \in \varX} \delta(\{\varx\}, \{\vary\})$.  Also by Claim~\ref{claim:induct}
\[ \delta(\varX, \varY'\{\vary\}) \leq \delta(\varX, \varY') + 2\delta(\varX, \varY'\{\vary\}), \]
so $\delta(\varX, \varY) \leq 2\sum_{\vary \in \varY} \delta(\varX, \{\vary\})$.  Combining the two inequalities we obtain the desired conclusion.
\end{proof}

To extend the proof to larger $k$ and obtain Claim~\ref{claim:kpart}, we generalize the first inequality in this sequence to $k$-partitions.  

\begin{claim}
\label{claim:refine}
For every $2k$-partition $(\varY_1, \dots, \varY_k, \varZ_1, \dots, \varZ_k)$,
\begin{align*}
& \delta(\varY_1, \dots, \varY_k, \varZ_1, \dots, \varZ_k) \leq 2\delta(\varY_1 \varZ_1, \dots, \varY_k\varZ_k) + 3\delta(\varY_1 \dots \varY_k, \varZ_1 \dots \varZ_k).
\end{align*}
\end{claim}
\begin{proof}
Assume that 
\begin{align*}
F(V) &\approx_\delta F_1(Y_1, Z_1) + \cdots + F_t(Y_t, Z_t) \\
F(V) &\approx_{\delta'} A(Y_1, \dots, Y_t) + B(Z_1, \dots, Z_t).
\end{align*}
By the triangle inequality
\begin{align*}
& A(Y_1, \dots, Y_t) + B(Z_1, \dots, Z_t) \approx_{\delta + \delta'} F_1(Y_1, Z_1) + \cdots + F_t(Y_t, Z_t).
\end{align*}
Fixing $Z_1, \dots, Z_t$ to values $\uZ_1, \dots, \uZ_t$ we get the decomposition
\begin{align*}
A(Y_1, \dots, Y_t) & \approx_{\delta + \delta'} F_1(Y_1, \uZ_1) + \dots + F_t(Y_t, \uZ_t) - B(\uZ_1, \dots, \uZ_t).
\end{align*}
and similarly
\begin{align*}
B(Z_1, \dots, Z_t) & \approx_{\delta + \delta'} F_1(\uY_1, Z_1) + \dots + F_t(\uY_t, Z_t) - A(\uY_1, \dots, \uY_t).
\end{align*}
Plugging these into the second equation gives the desired decomposition.
\end{proof}

\begin{proof}[Proof of Claim~\ref{claim:kpart}]
We assume that $k$ is a power of two and prove by induction that $\delta(\mathcal{P}) \leq c_k \Delta(\mathcal{P})$, where $c_k$ is the sequence $c_{2k} = 2c_k + 12$, $c_2 = 4$.  The base case $k = 2$ follows from Claim~\ref{claim:cuts}.  Assume the claim holds for $k$ and apply Claim~\ref{claim:refine} to $\mathcal{P}$. By inductive assumption and Claim~\ref{claim:cuts},
\[ \delta(\mathcal{P}) \leq 2 \cdot c_k\Delta(\varY_1 \varZ_1, \dots, \varY_k \varZ_k) + 3 \cdot 4\Delta(\varY_1 \dots \varY_k, \varZ_1 \dots \varZ_k). \]
Since $\mathcal{P}$ is a refinement of both these partitions, it follows that 
$\delta(\mathcal{P}) \leq (2c_k + 12)\Delta(\mathcal{P}) = c_{2k} \Delta(\mathcal{P})$, concluding the induction.  

The recurrence solves to $c_k = 8k - 12$, proving the claim when $k$ is a power of two.  When it is not, the same reasoning applies to the closest power of two exceeding $k$ (by taking some of the sets in the partition to be empty), which is at most $2k - 1$, proving the desired bound.
\end{proof}

\subsection{Hardness of exact variable bipartitioning}
\label{sec:vphard}

In contrast to Theorem~\ref{thm:first}, finding the exact bipartition is NP-hard even when there are only three variables.

\begin{proposition}
\label{prop:nphard1}
For any $n \geq 3$ there is no algorithm that outputs a bipartition of cost $\delta_2(F) + \eps$ over $\mathbb{Z}_2$ under Hamming distance in time polynomial in $\abs{\Sigma}/\eps$ with constant probability unless $BPP$ is in $NP$.
\end{proposition}
\begin{proof}
Assume such an algorithm $BIPARTITION$ exists. We show it can be used to solve the following problem:  Given explicit functions $F_1, \dots, F_{n-1}\colon \Sigma^2 \to G$, find $i^*$ that minimizes $\delta_2(F_{i^*})$ assuming this $i^*$ is unique, i.e. a function that has the smallest bipartition cost among the candidates.

Let $F(x_1, \dots, x_{n-1}, y) = F_1(x_1, y) + \cdots + F_{n-1}(x_{n-1}, y)$. The cost of the bipartition that splits $x_i$ from all other variables in $F$ is at most the cost of $F_i$, so $F$ has a bipartition of cost at least $\delta_2(F_{i^*})$.  We now argue that the cost of all other bipartitions is greater.  In fact, any other bipartition must split $y$ from $x_i$ for some $i \neq i^*$.  Assuming the cost of this partition is $\delta$, we must have
\[ F(X, Y) \approx_\delta A(X) + B(Y), \]
where $x_i \in X$ and $y \in Y$. After fixing all variables except for $x_i$ and $y$ we obtain
\[ F_i(x_i, y) \approx_\delta A(x_i, \underline{X_{-i}}) + B(y, \underline{Y_{-i}}) - \sum_{j \neq i} F_j(\underline{x_j}, y). \]
This is a bipartition for $F_i$ so its cost is strictly greater than $\delta_2(F_{i^*})$.  Therefore the output of  $BIPARTITION$ with oracle access to $F$ and $\eps = 1/\abs{\Sigma}^2$ has the desired property.

Roth and Viswanathan~\cite{roth-viswanathan} give an efficient reduction $R$ that maps a graph $G$ into a function $F$ such that if $G$ has a larger max-cut than $G'$ then $\delta_2(R(G)) < \delta_2(R(G'))$.  By composing the two reductions we obtain an efficient algorithm for deciding which of two graphs has a larger maximum cut, which is an NP-hard problem.  
\end{proof}

\section{Partitioning real-valued functions under the 2-norm}
\label{sec:second}

The problem of partitioning real-valued functions under the 2-norm is closely related to the well-studied problem of hypergraph partitioning.  To explain this connection we recall the Efron-Stein decomposition of real-valued functions over product sets.  The Efron-Stein decomposition of a function $F\colon \Sigma^n \to \R$ (under some product measure) is the unique decomposition of the form
\[ F(x) = \sum\nolimits_{S \subseteq [n]} \hat{F}_S \cdot F_S(x), \]
where $\hat{F}_S$ are real coefficients and $F_S$ are functions satisfying the following properties:
\begin{enumerate}
\item $F_S$ depends on the variables in $S$ only;
\item $\E[F_S | x_{S'}] = 0$ for any $S\not\subseteq S'$, where $x_{S'}$ is a fixing of all variables in $S'$;
\item $\E[F_S^2] = 1$.
\end{enumerate}
In particular, properties 1 and 2 imply that $\E[F_S F_T] = 0$ when $S \neq T$, and so $\E[F^2] = \sum_S \hat{F}_S^2$ by property 3.

\begin{proposition}
\label{prop:efronstein}
Given $F\colon \Sigma^n \to \R$, let $H$ be the hypergraph whose vertices are the variables of $F$ and whose hyperedges $S$ have weight $\hat{F}_S^2$ for every subset $S$.  The cost of the $k$-cut $(\varX_1, \dots, \varX_k)$ in $H$ equals $\delta_{\mathbb{R}, 2}(\varX_1, \dots, \varX_k)^2$.
\end{proposition}
\begin{proof}
We may assume $\E[F] = 0$ and use expression~\eqref{eq:deltaalt} to evaluate $\delta_{\mathbb{R}, 2}$.  The first term equals $\E[F^2] = \sum_S \hat{F}_S^2$.  The rest of the terms have the form $\E[\E[F | X_I]^2] = \E[F(X_{-I}, X_I) F(X'_{-I}, X_I)]$ for subsets $I$ of variables.   Plugging in the Efron-Stein decomposition of $F$ we have
\[ \E[\E[F | X_I]^2] = \sum_{S, T} \hat{F}_S \hat{F}_T \E[F_S(X_{-I}, X_I) F_T(X'_{-I}, X_I)]. \]
By property 2, the terms in the summation in which $S \neq T$ evaluate to zero.  Among the rest, if the set $S$ contains any variable $i$ outside $I$ then 
\[ \E[F_S(X_{-I}, X_I) F_T(X'_{-I}, X_I)] = \E\bigl[F_S(X_{-I}, X_I) \E[F_T(X'_{-I}, X_I) | X, X'_{-i}]\bigr] \]
and the inside expectation evaluates to zero by property 2.  Therefore the only surviving terms are those where $S = T$ and $S \subseteq I$, from where
\[ \E[\E[F | X_I]^2] = \sum\nolimits_{S \subseteq I} \hat{F}_S^2. \]
By~\eqref{eq:deltaalt},
\[ \delta_{\mathbb{R}, 2}(\varX_1, \dots, \varX_k) = \sum\nolimits_S \hat{F}_S^2 - \sum\nolimits_{i=1}^k \sum\nolimits_{S \subseteq \varX_i} \hat{F}_S^2 = \sum\nolimits_{\text{$S \not\subseteq \varX_i$ for any $i$}} \hat{F}_S^2. \]
The last quantity is the desired value of the cost of $k$-cut in $H$.
\end{proof}

When $\Sigma = \{-1, 1\}$ under uniform measure, the functions $F_S$ do not depend on $F$ and equal the Fourier characters $\chi_S(x) = \prod_{i \in S} x_i$, allowing us to embed instances of hypergraph partitioning into variable partitioning.

\begin{corollary}
\label{cor:hard2}
Assume there is an algorithm $A$ that given oracle access to $F\colon \{-1, 1\}^n \to \R$ under uniform measure outputs a $k$-variable partition of cost at most $C \cdot \delta_2(F) + \eps$ in time $t(n, k, \eps)$.  Then given a hypergraph with $n$ vertices and $m$ hyperedges with a $k$-cut of value $opt$, it is possible to output a $k$-cut of value $C \cdot opt$ in time $mn\cdot t(n, k, 1/m)$.
\end{corollary}

Chekuri and Li~\cite{chekuri-li} give a reduction from hypergraph $k$-cut to densest-$k$-subgraph. Manurangsi~\cite{manurangsi} shows that the latter is hard to approximate to within $\OO(n^{1/(\log \log n)^c})$ assuming the exponential-time hypothesis, implying inapproximability of the same order for $\delta_{\mathbb{R}, 2}(F)$.

On the positive side, Proposition~\ref{prop:efronstein} can be used to obtain variable partitioning algorithms from hypergraph partitioning ones.  The conversion is not direct as hypergraph partitioning assumes explicit access to the hypergraph. Klimmek and Wagner~\cite{klimmek-wagner} observed that submodularity of the hypergraph cut function $\xi(\varX) = \delta_{\mathbb{R}, 2}(\varX, \varXbar)^2$ allows for efficient minimization from {\em exact} oracle access.  To derive Theorem~\ref{thm:l2est} we extend the analysis to approximate oracle access.

The following proposition is an analysis of Queyranne's symmetric submodular minimization algorithm~\cite{queyranne} for an approximate input oracle.  We say $g$ is $\eps$-submodular if $g(\varX \varY \varZ) - g(\varX \varZ) - g(\varY\varZ) + g(\varZ) \leq \eps$ for all disjoint subsets $\varX, \varY, \varZ$.  

\begin{proposition}[Queyranne's algorithm with an approximate oracle]
\label{thm:queyranne}
There is an algorithm that given oracle access to a symmetric $\eps$-submodular $g$, makes $\OO(n^3)$ oracle queries and outputs a nontrivial subset $\varX$ such that $g(\varX)$ is within $n \eps/2$  of the minimum of $g$.
\end{proposition}

\begin{algorithm}[ht!]
\caption{Variable bipartitioning for real functions}
\begin{algorithmic}[1]
\State {\bfseries Output:} partition $\mathcal{P}$
\State Define random function $\xi(\varX)$ to be an estimation of $\EE[D_F(\varX, \varXbar)^2]$;
\State Run symmetric submodular minimization over $\varX$ on $\xi(\varX)$;\footnotemark{}
\State Output partition $\mathcal{P}=(\varX, \varXbar)$;
\end{algorithmic}
\label{algo-realf2}
\end{algorithm}

\footnotetext{For example, Queyranne's algorithm \cite{queyranne} solves symmetric submodular minimization.}

\begin{namedtheorem}[Theorem \ref{thm:l2est}]
Let $F\colon \Sigma^n \to \R$ be a function with $\norm{F}_{\R, 4} \leq 1$.  There is an algorithm that given inputs $n$, $\eps$, $\gamma$, and oracle access to $F$, runs in time $\OO(n^5\log(n/\gamma)/\eps^2)$ and outputs a bipartition $(\varX, \overline{\varX})$ such that $\delta_{\mathbb{R}, 2}(\varX, \overline{\varX})^2 \leq \delta_{\mathbb{R}, 2}(F)^2 + \eps$ with probability at least $1 - \gamma$.
\end{namedtheorem}

\begin{proof}[Proof of Theorem~\ref{thm:l2est}]
By Proposition~\ref{prop:l2given}, $\delta_{\mathbb{R}, 2}(\varX, \overline{\varX})^2$ can be estimated to within error $\eps/Kn$ with $\OO(\log(n / \gamma) n^2/\eps^2)$ queries to $F$ with probability $1 - K \gamma / n^3$ for any constant $K$.  This estimator implements an $\eps/2n$-approximate oracle to $\delta_{\mathbb{R}, 2}(\varX, \overline{\varX})^2$ with probability $1 - \gamma$ with respect to an algorithm that makes at most $Kn^3$ queries.  In particular, with probability $1 - \gamma$, the output of the oracle is $\eps/4n$-close to the value of the submodular function $\delta_{\mathbb{R}, 2}^2$ at all points queried by Queyranne's algorithm and also at the minimum of $\delta_{\mathbb{R}, 2}^2$.  Since from the algorithm's perspective it is interacting with a symmetric $\eps/n$-submodular function $g$, it outputs a partition such that $g(\varX, \overline{\varX})$ is within $\eps/2$ of the minimum of $g$.  By the triangle inequality, $\delta_{\mathbb{R}, 2}(\varX, \overline{\varX})^2$ is within $\eps/2 + 2\eps/4n \leq \eps$ close to the minimum of $\delta_{\mathbb{R}, 2}^2$.
\end{proof}

Saran and Vazirani's approximation algorithm \cite{saran1995finding} for multiway $k$-cut (with fixed terminals) can be viewed as a reduction from multiway $k$-cut to multiway $2$-cut.  The reduction works given access to approximate $s$-$t$-cut oracles, where $s$ and $t$ are designated terminals that must be split by the cut.  The corresponding cut function $\delta_{\mathbb{R}, 2}(\bm{s}\varX, \bm{t}\varXbar)$, where $(\varX, \varXbar)$ is now a partition of $\varV - \{\bm{s}, \bm{t}\}$, is still submodular but no longer symmetric. 

Therefore, we desire an analogue of Proposition~\ref{thm:queyranne} for general (not necessarily symmetric) submodular minimization~\cite{grotschel-lovasz-schrijver, iwata-orlin}. Blais et al. \cite{blais2018tolerant} (Algorithm 2 and Corollary 5.4 in their paper) have proposed such an algorithm under the context of tolerant junta testing, by investigating the Lovász extension \cite{grotschel-lovasz-schrijver} and a separation oracle \cite{lee2015faster} for the optimization. Given an inexact oracle to a submodular function up to $\poly(\eps/n)$ estimation error, they provide an algorithm to minimize the function up to an $\eps$ optimization error in time $\poly(n/\eps)$, leading to our algorithm for $k$-partitioning.


\begin{corollary}
\label{cor:sarvaz}
Let $F\colon \Sigma^n \to \R$ be a function with $\norm{F}_{\R, 4} \leq 1$ and $\norm{F}_{\infty} \leq 1$.  There is an algorithm that given inputs $n$, $k$, $\eps$, $\gamma$, and oracle access to $F$, runs in time $\OO(k^2 n^k\poly(n/\eps)\log(1/\gamma))$ and outputs a $k$-partition $\mathcal{P}$ such that $\delta_{\mathbb{R}, 2}(\mathcal{P})^2 \leq (2 - 2/k)\delta_{\mathbb{R}, 2}(F)^2 + \eps$ with probability $1 - \gamma$.
\end{corollary}

\begin{algorithm}[ht!]
\caption{Variable $k$-partitioning for real functions}
\begin{algorithmic}[1]
\State {\bfseries Input:} number of sets $k$ in the partition, $k\ge 3$
\State {\bfseries Output:} partition $\mathcal{P}$
\State Define random function $\xi'(\varX,\bm{s}, \bm{t})$ to be an estimation of $\EE[D_F(\bm{s}\varX, \varV\setminus \bm{s}\varX)^2]$, where $\bm{s}, \bm{t}\notin \varX$;
\For{Set $\bm{W}$ of $k$ vertices out of $\binom{n}{k}$ combinations}
\State Run multiway-$k$-cut given $k$ terminals $\bm{W}$ and obtain $\mathcal{P}_{\bm W}$, where the cost of $\bm s$-$\bm t$-cut is treated as $\min_{\varX}\xi'(\varX,\bm{s}, \bm{t})$ for any $s,t\in \varW$;\footnotemark{}
\EndFor
\State Output the partition $\mathcal{P}_{\bm W}$ with the minimum cost over all $\bm{W}$;
\end{algorithmic}
\label{algo-realfk}
\end{algorithm}

It remains to prove Proposition~\ref{thm:queyranne}.  

\begin{claim}
\label{claim:queyranne}
Let $g$ be $\eps$-submodular.  Assume there exists $\varx \in \varW$ such that for all $\varY \subseteq \varW \setminus \varx$ and $\varu \not\in \varW$,
\[ g(\varW) + g(\varu) \leq g(\varW \setminus \varY) + g(\varY\varu) + \delta. \]
If $\varx'$ maximizes $g(\varW \varu) - g(\varu)$ among all $\varu \not\in \varW$ then
\[ g(\varW\varx') + g(\varu) \leq g(\varW\varx' \setminus \varY) + g(\varY\varu) + (\delta + \eps). \]
\end{claim}
\begin{proof}
If $\varx \not\in \varY$ then
\begin{align*}
g(\varW \varx') + g(\varu) 
  &\leq \bigl(g(\varW) - g(\varW \setminus \varY) + g(\varW\varx' \setminus \varY)\bigr) + g(\varu) + \eps 
  &&\text{by $\eps$-submodularity} \\
  & = g(\varW\varx' \setminus \varY) + \bigl(g(\varW) - g(\varW \setminus \varY) + f(\varu)\bigr) + \eps \\
  &\leq g(\varW\varx' \setminus \varY) + g(\varY\varu) + (\delta + \eps) 
  &&\text{by inductive hypothesis.}
\end{align*}
Otherwise, $\varx \not\in \varW \setminus \varY$ and
\begin{align*}
g(\varW \varx') + g(\varu) 
  &\leq g(\varW \varu ) + g(\varx')   &&\text{by optimality of $\varx'$} \\
  &\leq \bigl(g(\varW) - g(\varY) + g(\varY\varu)\bigr) + g(\varx') + \eps   &&\text{by $\eps$-submodularity} \\
  &= \bigl(g(\varW) + g(\varx') - g(\varY)\bigr) + g(\varY\varu) + \eps \\
  &\leq g(\varW \varx' \setminus \varY) + g(\varY\varu) + (\delta + \eps) 
  &&\text{by inductive hypothesis.} \hfill \tag*{\qedhere}
\end{align*}
\end{proof}

\footnotetext{Multiway-$k$-cut can be solved by, for example, Saran and Vazirani's algorithm \cite{saran-vazirani}. The function $\xi'(\varX,\bm{s}, \bm{t})$ is still submodular over $\varX$ but not necessarily symmetric. The value $\min_{\varX}\xi'(\varX,\bm{s}, \bm{t})$ can be computed by general submodular minimization algorithms like \cite{grotschel-lovasz-schrijver} and \cite{iwata-orlin}.}

\begin{proof}[Proof of Proposition~\ref{thm:queyranne}]
Queyranne's algorithm $Q^g$ is recursive.  If $n = 2$ the unique partition is output.  Otherwise, starting from an arbitrary singleton set $\varW_1$, the algorithm sets $\varW_{i+1} = \varW_i\varx_i$, where $\varx_i$ maximizes $g(\varW_i\varu) - g(\varu)$ among all $\varu \not\in \varW_i$.  The algorithm then contracts the elements $\varx_{n-1}$ and $\varx_n$ into $\varx_{n-1}\varx_n$ and outputs the smaller value of $Q^g(\varx_1, \dots, \varx_{n-2}, \varx_{n-1}\varx_n)$ and $g(\varx_n)$.

We prove by induction on $n$ that the output of $Q^g$ is $(n - 1)\eps/2$-close to the minimum of $g$.  The base case $n = 2$ is clear.  Now assume this is true for inputs of size $n - 1$.  If the minimum partition of $g$ doesn't split $\varx_{n-1}$ and $\varx_n$ then the claim follows by inductive assumption.  

Otherwise, we show that $g(\varx_n)$ is within $(n - 1)\eps/2$-close to the minimum of $g$. Applying Claim~\ref{claim:queyranne} iteratively to the sets $\varW_1, \dots, \varW_{n-1}$, we conclude that
\[ g(\varW_{n-1}) + g(\varx_n) \leq g(\varW_{n-1} \setminus \varY) + g(\varY\varx_n) + (n - 1)\eps \]
for all $\varY$ that do not contain $\varx_n$ and $\varx_{n-1}$.  Applying symmetry this inequality can be rewritten as $g(\varx_n) \leq g(\varx_n Y) + (n - 1)\eps/2$.  As $\varx_{n-1}$ and $\varx_n$ are split in the optimal solution it must be of type $\varx_n Y$ for some $Y$ excluding $\varx_{n-1}$, so $g(\varx_n)$ is $(n-1)/2\eps$ close to the minimum as desired.
\end{proof}

\section{Testing partitionability}
\label{sec:testing}

As a consequence of Theorem~\ref{thm:first}, $k$-partitionability is testable with $\tilde{\OO}(k^{2p} n^{4p + 2}/\eps^{2p})$ queries.  The yes instances are inputs with $\delta_k(F) = 0$; the no instances are inputs with $\delta_k(F) > \eps$.  The query complexity of Theorem~\ref{thm:tester} can be obtained by the improved analysis that follows.  

To simplify notation we only prove the theorem for the Hamming weight over $\mathbb{Z}_q$ and describe the change necessary for $p$-norms over $\R$.  

\begin{algorithm}[!htbp]
\caption{Tester for $k$-partitionability}
\begin{algorithmic}[1]
\State Create an empty undirected graph $G$ with vertex set $\varV$;
\For{$\OO(kn/\eps)$ times}
\For{For every pair of distinct variables $\varx, \vary\in \varV$}
\State Choose random $x, y, x', y', Z$;
\State If $F(x, y, Z) + F(x', y', Z) \neq F(x, y', Z) + F(x', y, Z)$ create the edge $\{\varx, \vary\}$ in $G$;
\EndFor
\EndFor
\State Accept if the graph has at least $k$ connected components;
\end{algorithmic}
\label{algo:tester}
\end{algorithm}

We will need the following fact:

\begin{fact}
\label{fact:ind}
If $p_1, \dots, p_n$ are probabilities such that $\sum p_i \geq \eps$ then $1 - \prod (1 - p_i) \geq \eps - \OO(\eps^2)$. 
\end{fact}
\begin{proof}
Using the inequality $1 - p \leq e^{-p}$ and the second-order estimate $e^{-x} = 1 - x + \OO(x^2)$, we have
\[ \prod (1 - p_i) \leq \prod e^{-p_i} = e^{-\sum p_i} \geq e^{-\eps} = \eps - \OO(\eps^2). \hfill \tag*{\qedhere} \]
\end{proof}

\begin{namedtheorem}[Theorem \ref{thm:tester}]
$k$-partitionability is testable with one-sided error and $\OO(kn^3/\eps)$ non-adaptive queries with respect to Hamming weight over $\mathbb{Z}_q$, and with $\OO(k^{2p}n^3/\eps^{2p})$ non-adaptive queries with respect to the $p$-norm over $\mathbb{R}$ assuming $\norm{F}_{2p} \leq 1$.
\end{namedtheorem}

While in this work we are mainly interested in small values of $k$, in the extreme case when $k = n$ the partition is unique and the property is testable with $\OO(1/\eps)$ queries by the result of Dinur and Golubev~\cite{dinur-golubev}.

\begin{proof}[Proof of Theorem~\ref{thm:tester}]
The connected components of $G$ are always contained in the partition components of $F$, so if $F$ is $k$ partitionable the tester always accepts.  We argue that with constant probability, all bipartition of $F$ satisfying $\delta(\varX, \bm{\overline{X}}) \geq \eps$ cross an edge in $G$.  

If $F$ is $\eps$-far from $k$-partitionable, by Claim~\ref{claim:kpart} $\Delta(\mathcal{P}) = \Omega(\eps/k)$ for all $k$-partitions $\mathcal{P}$.  As every $k$-cut can be coarsened into a 2-cut of at least half the weight, every $k$-partition can be coarsened into a bipartition such that $\Delta(\varX, \overline{\varX}) = \Omega(\eps/k)$.  We now argue that with constant probability, all such heavy bipartitions $(\varX, \overline{\varX})$ are crossed by an edge in $G$, so no $k$-partition is likely to survive in $G$.

Assume $\Delta(\varX, \overline{\varX}) = \sum_{\varx \in \varX, \vary \in \bm{\varX}} \norm{D_F(\varx, \vary)} = \Omega(\eps/k)$.  As $\norm{D_F(\varx, \vary)}$ is the acceptance probability of the test in line 5, by Fact~\ref{fact:ind} in any given iteration of the outer loop 3 at least one of these edges will appear in $G$ with probability $\Omega(\eps/k)$.  (For $p$-norms over $\R$, $\norm{D_F(\varx, \vary)}$ is an expectation that takes $\OO((\eps/k)^{2p})$ queries to estimate.)  After $\OO(kn/\eps)$ iterations the probability that the cut survives is less than $2^{-n}$.  By a union bound the probability that any heavy cut survives is at most half.
\end{proof}

It is not difficult to see that $n$ queries are required for one-sided error testers when $k$ equals 2 as the relevant constraints span an $n$-dimensional space. We show that a linear dependence on $n$ is necessary for two-sided error testers as well.  Proposition~\ref{prop:lower1} shows a general $\Omega(n)$ lower bound for functions over finite domains (with uniform measure) valued over finite groups under the Hamming metric.  Proposition~\ref{prop:lower2} shows that $\Omega(n)$ {\em non-adaptive} queries are necessary for functions from $\R^n$ to $\R$ under the 2-norm.

\begin{proposition}
\label{prop:lower1}
Testing 2-partitionability for functions $F\colon \mathbb{Z}_q^n \to G$ for a finite group $G$ under uniform measure and Hamming metric requires $\Omega(n - k)$ queries even for constant $\eps$.
\end{proposition}

For simplicity of notation we present the proof in the case $q = 2$ and $G = \mathbb{Z}_2$.  The proof is closely related to Chockler and Gutfreund's lower bound for testing juntas~\cite{chockler-gutfreund}.

\begin{proof}
Let $R\colon \mathbb{Z}_2^n \to \mathbb{Z}_2$ be a random function and $P\colon \mathbb{Z}_2^{n-1} \to \mathbb{Z}_2$ be a function that depends on all but a random hidden input coordinate $I$.  First we argue that $\delta_2(R) = \Omega(1)$ with high probability.   For this it is sufficient to argue that $\norm{D_R(\varX, \varY)} = \Omega(1)$ for every partition $(\varX, \varY)$.  By definition $\norm{D_R(\varX, \varY)}$ is the average value of $\Omega(2^{2n})$ indicator values for events of the type $R(x, y) + R(x', y') - R(x', y) - R(x, y') = 0$.  These events have probability half each and are pairwise independent, so by Chebyshev's inequality the probability that the $\norm{D_R(\varX, \varY)}$ is sub-constant is $\Omega(2^{-2n})$.  Taking a union bound over all $2^n$ bipartitions it follows that $\norm{D_R(\varX, \varY)} = \Omega(1)$ with probability at least $1 - \Omega(2^{-n})$.

To complete the proof, it is sufficient to argue that with high probability any $Q$ queries to $P$ reveal independent random bits.  Consider the subspace of $\mathbb{Z}_2^n$ spanned by the $Q$ queries (or the submodule of $\mathbb{Z}_q^n$ if $q$ is not a prime).  This vector space has dimension at most $Q$, so it can contain at most $Q$ of the elementary basis vectors $e_1, \dots, e_n$.  However, unless it contains $e_I$ for the hidden coordinate $I$, no two queries differ in a single coordinate and all answers are independent random bits.  Since $I$ is uniformly random the probability that $P$ and $R$ can be distinguished is at most $Q/n$.  By a union bound the distinguishing advantage of the tester is at most $\Omega(2^{-n}) + Q/n$, which is subconstant unless $Q = \Omega(n)$.
\end{proof}

It was pointed out to us by Guy Kindler that the proof of Proposition~\ref{prop:lower1} to functions from $\R^n$ to $\R$ say under Gaussian measure by considering the functions $R_{\mathbb{R}}$ and $R_{\mathbb{P}}$ given by $F_\mathbb{R}(x_1, \dots, x_n) = F(\mathrm{sign}\ x_1, \dots, \mathrm{sign}\ x_n)$ where the sign is interpreted as a Boolean value.  This example is somewhat unnatural because the functions are discontinuous.  The following proposition shows that testing still requires $\Omega(n)$ non-adaptive queries even for highly smooth functions.  

\begin{proposition}
\label{prop:lower2}
Testing $2$-partitionability non-adaptively for quadratic functions from $\R^n$ to $\R$ under the 2-norm requires $\Omega(n)$ queries under any measure with zero mean and unit variance and bounded third and fourth moments. 
\end{proposition}

We need the following claim about distinguishing linear functions of normal random variables.

\begin{claim}
\label{claim:nstat}
Let $Z_1, \dots, Z_n$ be independent standard normal random variables, $F(x) = \sum_{i = 1}^n Z_i x_i$, and $F'(x) = \sum_{i \in S} Z_i x_i$ where $S \subseteq [n]$ is a random subset of size $s$.  For any $q$ queries $x^1, \dots, x^q \in \R^n$, $(F(x^1), \dots, F(x^q))$ and $(F'(x^1), \dots, F'(x^q))$ are $\OO(qs/(n - s + 1))$-statistically close. 
\end{claim}
\begin{proof}
It is sufficient to prove the claim for $\abs{S} = n - 1$ and apply the triangle inequality.  By convexity it is sufficient to upper bound the expected statistical distance averaged over the choice of the index $i$ that is omitted from $S$.  For fixed $i$, since the queried functions are linear, without loss of generality we may assume that the queries $x^1, \dots, x^q$ are orthonormal.  Let $X$ be the $q \times n$ matrix whose rows are the queries $x^1, \dots, x^n$, and $X_{-i}$ be the submatrix obtained by removing the $i$-th column.  The desired statistical distance is then  within a constant factor of $\norm{(X^TX)^{-1}(X_{-i}^T X_{-i}) - I}_F$, where $\norm{\cdot}_F$ is the Frobenius norm~\cite{barsov-ulyanov}.  By orthonormality we obtain that $\norm{(X^TX)^{-1}(X_{-i}^T X_{-i}) - I}_F = \norm{x^i}_2^2$.  Averaging over $i$, the desired statistical distance is at most $\OO(\E_i[\norm{x^i}_2^2]) = \OO(q/n)$.
\end{proof}

\begin{proof}[Proof of Proposition~\ref{prop:lower2}]
Let $F(x) = n^{-1} \sum_{j \neq k} Z_{jk} x_jx_k$, where $Z_{jk}$ are independent standard normal random variables.  Let $F'(x) = n^{-1} \sum_{\text{$j, k, i$ distinct}} Z_{jk} x_jx_k$ where $i$ is chosen at random from $[n]$.  By standard concentration inequalities both $\norm{F}_4$ and $\norm{F'}_4$ are constant with high probability.  By Claim~\ref{claim:nstat}, the answers to any $q$ non-adaptive queries to $F$ and $F'$ are $\OO(q/n)$-statistically close.  

It remains to argue that $F$ is $\Omega(1)$-far from 2-partitionable.  For a fixed bipartition $(S, T)$ of $[n]$, by Claim~\ref{claim:nstat} the cost of $F$ is $n^{-1} \sum_{j \in S, k \in T} Z_{jk}^2$.  Therefore the average cost (over the randomness of $F$) is $\abs{S}\abs{T}/n$.  By Chernoff bound the cost is at least $\Omega(\abs{S}\abs{T}/n) = \Omega(1)$ with probability $1 - \exp(-\Omega(\abs{S}\abs{T}))$.  Taking a union bound over all $2^{n-1}$ possible bipartitions we conclude that $F$ is $\Omega(1)$-far from 2-partitionable with probability $1 - \exp(-\Omega(n))$.
\end{proof}

If Claim~\ref{claim:nstat} extends to adaptive queries, so would Proposition~\ref{prop:lower2}.

\section{Applications to reinforcement learning control}
\label{sec:apprl}

In this section we discuss the application of variable partitioning algorithms given real-valued oracle and the 2-norm measure. The set $\varX$ of variables to be partitioned corresponds to the set $a$ of control variables (the action), while the oracle $F$ corresponds to the advantage function $A$. While the control task is achieved by a series of actions, the advantage function describes how much one single action in the series can affect the final objective. In general, this function is complex enough so explicit representation is not available. 
Instead, it is usually estimated by Monte-Carlo sampling of the action series or by function approximation, where in either case it is sensible to treat the function as an oracle.

In a reinforcement learning control task, the objective is to control the action $a$ so as to maximize the expected cumulative reward over time $t$. The advantage function $A(\cdot,a)$ describes the marginal gain of such an objective of an action $a=a_t$ at time $t$. This function can be estimated by Monte-Carlo sampling of the actions $a_t,a_{t+1},\dots$, or by function approximation. In either of the cases it is sensible to treat the function as an oracle when using it to partition the variables.

We compare empirically with three previous approaches. The first approach is a standard approach proposed by Williams \cite{williams1992simple,sutton2018reinforcement} and later improved by Mnih et al. \cite{mnih2016asynchronous} and Schulman et al. \cite{schulman2017proximal}. These approaches learn reinforcement learning control without considering the possible partitioning of the advantage function. The second approach is to trivially partition $n$ variables into $n$ subsets \cite{wu2018variance,pytorchrl}. This causes a large partition error which induces bias in the learning update. The third baseline partitions the variables heuristically \cite{li2018policy}. In their method the Hessian matrix of the advantage function is first calculated using a discrete gradient method. Then this Hessian matrix is treated as an adjacency matrix of a graph, by the heuristic that two independent variables have a zero element in Hessian. Then elements are removed from Hessian, from those with the lowest absolute values, until the graph has at least $k$ connected components. This algorithm shares some similar intuition with our first algorithm.

The rest of this section will introduce the preliminaries of how this partition may be used in reinforcement learning, and then demonstrate the comparison of scores attained in the experiments.

\subsection{Reinforcement learning control and policy gradient}
\label{sec:pre}

We consider a reinforcement learning task described by a discrete-time Markov decision process (MDP), denoted as the tuple $(\mathcal{S},\mathcal{A}, \mathcal{T}, r,\rho_0,\beta)$. That includes $\mathcal{S} \in \mathbb{R}^m$ the $m$ dimensional state space, $\mathcal{A} \in \mathbb{R}^n$ the $n$ dimensional action space, $\mathcal{T}:\mathcal{S}\times\mathcal{A}\times\mathcal{S} \to \mathbb{R}^+$ the environment transition probability function, $r:\mathcal{S}\times\mathcal{A}\to\mathbb{R}$ the reward function, $\rho_0$ the initial state distribution and $\beta \in [0,1)$ the unnormalized discount factor. Here $n$ is the number of the control variables, which is consistent with the dimension of the input oracle. A (stochastic) policy is a function $\pi:\mathcal{S}\times\mathcal{A}\to\mathbb{R}_+$ that outputs a distribution over $\mathcal{A}$ on any given state $s\in\mathcal{S}$. The objective of reinforcement learning is to learn a policy $\pi$ such that the expected cumulative reward
$J(\theta) = \EE_{s\sim\rho_\pi, a \sim \pi}[\sum_{t=0}^\infty \beta^tr(s_t,a_t)],$
is maximized, where $\rho_\pi(s) =\sum_{t=1}^\infty \beta^{t-1} \PP(s_t =s)$. Since $\pi$ is in a functional space, the problem is commonly relaxed to find over the space of parameterized functions the policy, such as the space of neural networks. When the policy is parameterized we denote it as $\pi_\theta$.

Advantage actor-critic (A2C), a standard approach in policy optimization \cite{mnih2016asynchronous,schulman2017proximal}, estimates the gradient of the policy $\nabla_\theta J(\theta)$. According to the policy gradient theorem \cite{williams1992simple}, this gradient can be estimated by $\nabla_\theta J(\theta) = \EE_{\pi(a|s)}[\nabla_\theta\log\pi(a|s)A^\pi(s,a)]$, where $A^\pi(s,a)$ is the advantage function of $s$,$a$, and policy $\pi$. Here $A^\pi(s,a)$ is defined as $A^\pi(s,a) = Q^\pi(s,a)-V^\pi(s)$, where $Q^\pi(s,a) = \EE_{\pi}[\sum_{t'\geq t}^\infty \beta^{t'-t}r(s_{t'},a_{t'})|s=s_t,a=a_t,\pi]$ is the action-state value function and $V^\pi(s)=\EE_{a\sim \pi(a|s)}[Q^\pi(s,a)]$ the state-value function.

It is shown later in \cite{wu2018variance} and \cite{li2018policy}, that an alternative estimator
\begin{align}
\label{gas}
\nabla_\theta J(\theta) & = \sum\limits_{j=1}^k \EE_{\pi(a_{(j)}|s)}[\nabla_\theta\log\pi(a_{(j)}|s)(A^\pi(s,a_{(j)})],
\end{align}
may induces a lower variance. The condition that this estimator holds is that the advantage function can be approximately partitioned into $k$ parts correspondingly:
\[
A^\pi(s,a) = A_1^\pi(s,a_{(1)}) + \cdots + A_k^\pi(s,a_{(k)}) + U(s,a)
\]
for some state $s$ the estimation takes place, where $U(s,a)$ the partition error is expected to be small for the estimator to be accurate.

The learning is an iterative process that takes $N$ updates by the gradient $\nabla_\theta J(\theta)$ while the $k$-partition is computed every $N/N_1$ iterations. Every run of the partitioning algorithm outputs the disjoint subsets $a_{(1)},\dots, a_{(k)}$, which is then used by \eqref{gas} for $N_1$ iterations. It is worth note that our algorithm has a complexity of $\OO(N_1n^5)$, which is negligible in reinforcement learning. As the Monte-Carlo estimation of $\nabla_\theta J(\theta)$ requires a complete trial of the task (for example, play a game for an entire episode), which involves the interaction of a complex system.

\subsection{Experiments}

We compare our first algorithm (called pairwise estimates - PE) and our second algorithm (called submodular minimization - SM) with the aforementioned existing approaches. A2C \cite{mnih2016asynchronous} is the baseline approach in reinforcement learning who does not leverage variable partition. It uses control variates (CV) as the primary variance reduction technique. Other methods \cite{wu2018variance,li2018policy} partition the control variable so as to reduce the variance in the Monte-Carlo estimation by Rao-Blackwellization (RB) \cite{casella1996rao}. For the discussion on variance reduction we refer the readers to the paper cited above. The comparisons are summarized below

\begin{table}[ht!]
\label{tbl:compare}
\centering
\begin{tabular}{llllll}
\toprule
PG estimator & Variance reduction & Heuristics & Partitioning & Guarantees & Limits \\
\midrule
A2C \cite{mnih2016asynchronous} & CV & - & - & - & - \\ 
Wu et al. \cite{wu2018variance} & CV \& RB & yes & fully & no & $k=n$ \\
Li and Wang \cite{li2018policy} & CV \& RB & yes & greedy & no & \textbf{{no}} \\
PE (our first) & CV \& RB & \textbf{{no}} & greedy & factor-$\OO(kn^2)$ & \textbf{{no}} \\
SM (our second) & CV \& RB & \textbf{{no}} & \textbf{{optimal}} & \textbf{{almost opt}} & \textbf{{no}}\\
\bottomrule
\end{tabular}
\caption{Comparisons of our algorithms with previous ones}
\end{table}

Now we study the performance in terms of both the correctness and the optimality on graph cuts on weighted graphs. Correctness notes the number of times the algorithm outputs exactly the optimal partition, while optimality describes the average of the ratio of the partition error and the optimal partition error, over all the independent runs. This will illustrate the difference between greedy-based algorithms like \cite{li2018policy} and our first algorithm, and submodular minimization based algorithms like our second algorithm. Note that submodular minimization always finds the optimal partition. 

\begin{table}[ht!]
\label{tbl:variable1}
\centering
\begin{tabular}{llllll}
\toprule
\#Nodes $n$ & $n=5$ & $n=10$ & $n=20$ & $n=40$ & $n=100$ \\
\midrule
Submodular & - & - & - & - & - \\ 
Greedy (correctness) & 7753 & 6271 & 4226 & 2380 & 1101 \\
Greedy (optimality) & 1.060 & 1.203 & 1.408 & 1.352 & 1.250 \\
\bottomrule
\end{tabular}
\caption{Performance of the greedy algorithm on variable partition}
\end{table}

Then Table \ref{tbl:variable2} compares the partitioning algorithms when the oracle is a quadratic function $a^TH_0a$ for some random $H_0$. In this case our second algorithm SM also incurs an error per Theorem \ref{thm:l2est}, but the error in practice is shown to be small enough. It has constantly the best empirical performance in both correctness and optimality.

Since we only replaced heuristic partitioning with our partitioning algorithm in reinforcement learning, it is reasonable that our more accurate partitions will improve reinforcement learning.

\begin{table}[th!]
\centering
\begin{tabular}{llllll}
\toprule
\#Nodes $n$ & $n=5$ & $n=10$ & $n=20$ & $n=40$ & $n=100$ \\
\midrule
Li and Wang \cite{li2018policy} (correctness) & 7553 & 5651 & 2929 & 1251 & 400 \\
PE (correctness) & 7709 & 6108 & 4001 & 2020 & 918 \\
SM (correctness) & \textbf{9896} & \textbf{9630} & \textbf{9243} & \textbf{8193} & \textbf{6802}\\
\midrule
Li and Wang \cite{li2018policy} (optimality) & 1.150 & 1.281 & 1.508 & 1.501 & 1.290 \\
Wu et al. \cite{wu2018variance} (optimality) & 9.049 & 13.54 & 20.96 & 34.42 & 72.55 \\
PE (optimality) & 1.075 & 1.277 & 1.452 & 1.400 & 1.281 \\
SM (optimality) & \textbf{1.020} & \textbf{1.028} & \textbf{1.101} & \textbf{1.110} & \textbf{1.025}\\
\bottomrule
\end{tabular}
\caption{Comparisons of the algorithms on variable partition}
\label{tbl:variable2}
\end{table}

Finally we plug our algorithms into reinforcement learning control, replacing the partitioning steps in \cite{li2018policy}. The tasks we are testing on are standard tasks in reinforcement learning by the MuJoCo physics simulator. This includes training a simplified model of ant, cheetah, or human to run as fast as possible. The score is the cumulative reward over time, where the reward is the speed less the energy cost (which is $0.001\|a\|_2^2$). The control variables $a$ are the forces applied to the joints. We refer to \cite{brockman2016openai} for the exact simulator settings.

We have conducted experiments on all eight environments from MuJoCo that has the action dimensional higher than one, shown in Figure \ref{fig:mujoco} below. In the figure the $x$-axis is the number of Monte-Carlo sample updates, which can be regarded as the time elapsed on the training, while the $y$-axis is the score attained by the model. Our second algorithm (SM) has achieved the highest score among most of these tasks, which agrees with our theoretical finding.

\begin{figure*}[ht!]
\centering 
{\includegraphics[width=0.24\textwidth]{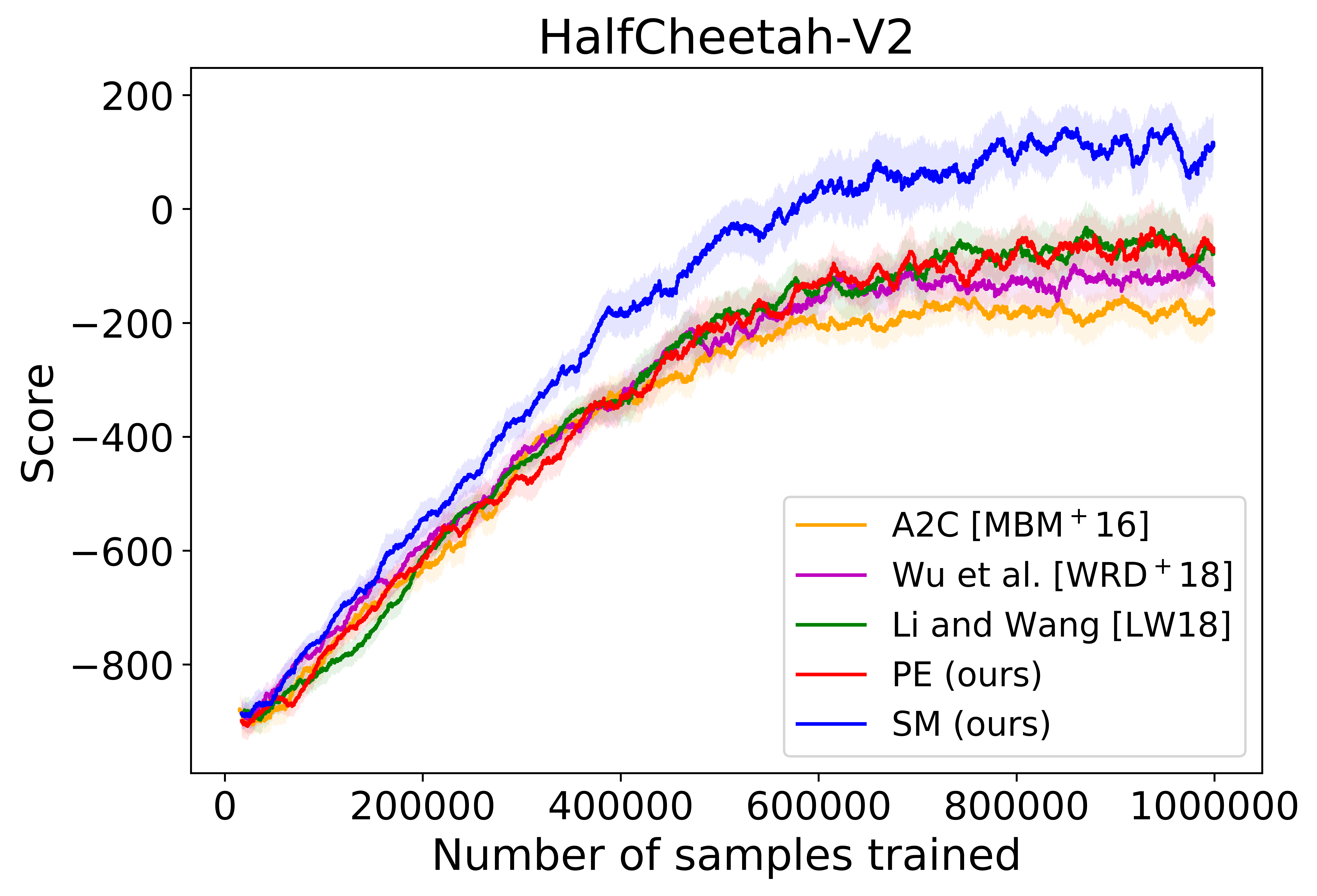}}
{\includegraphics[width=0.24\textwidth]{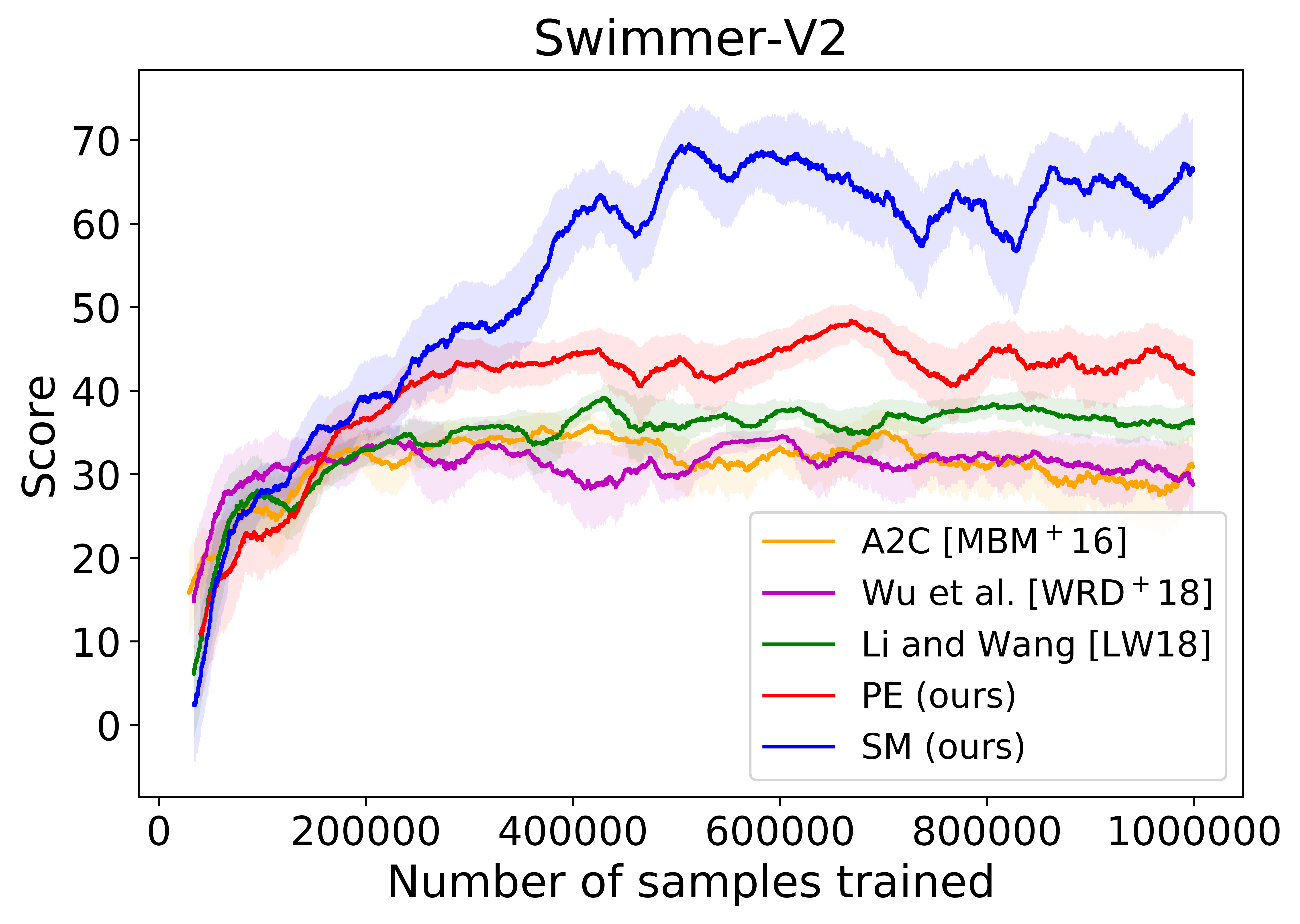}} 
{\includegraphics[width=0.24\textwidth]{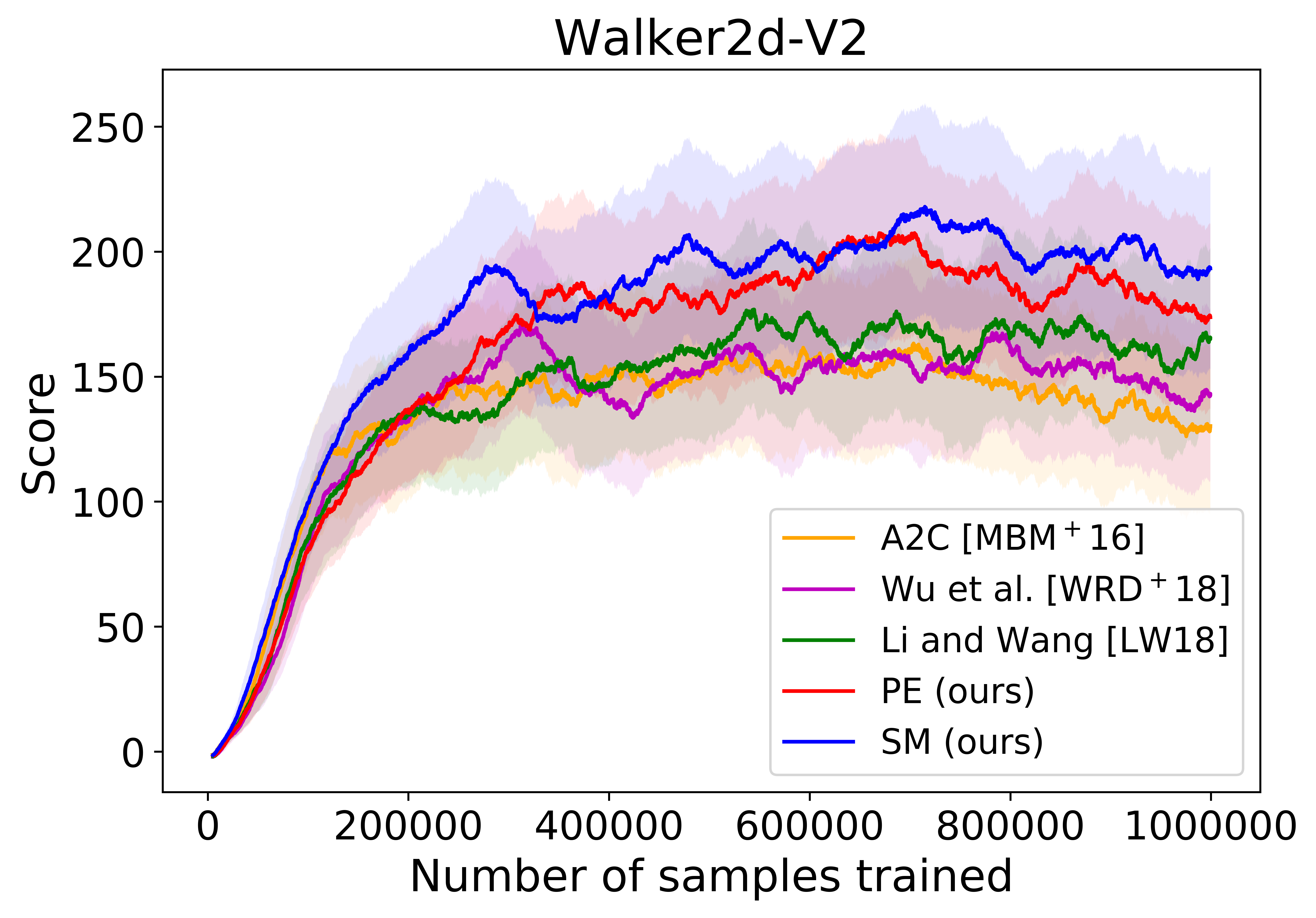}}
{\includegraphics[width=0.24\textwidth]{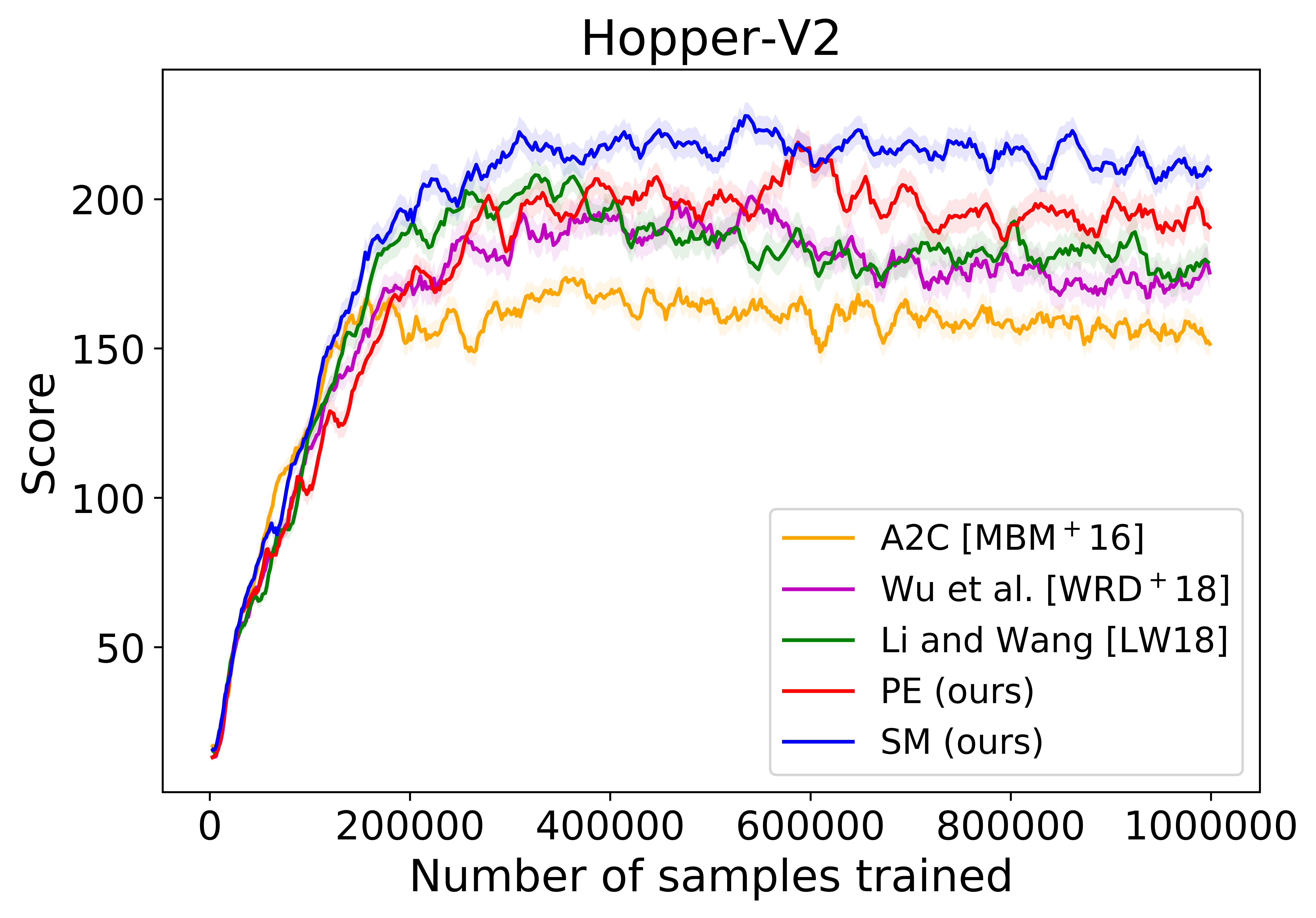}}\\
{\includegraphics[width=0.24\textwidth]{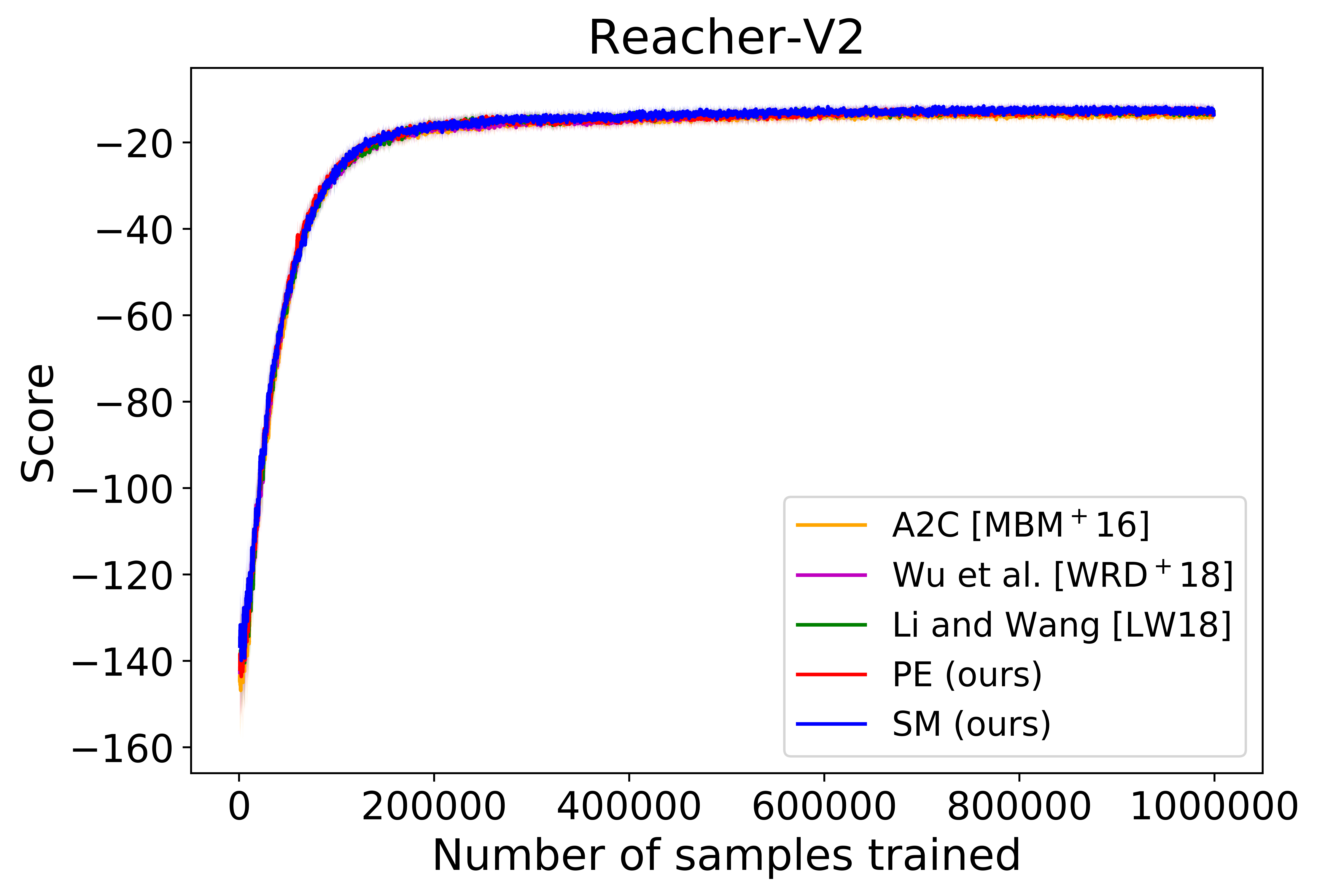}}
{\includegraphics[width=0.24\textwidth]{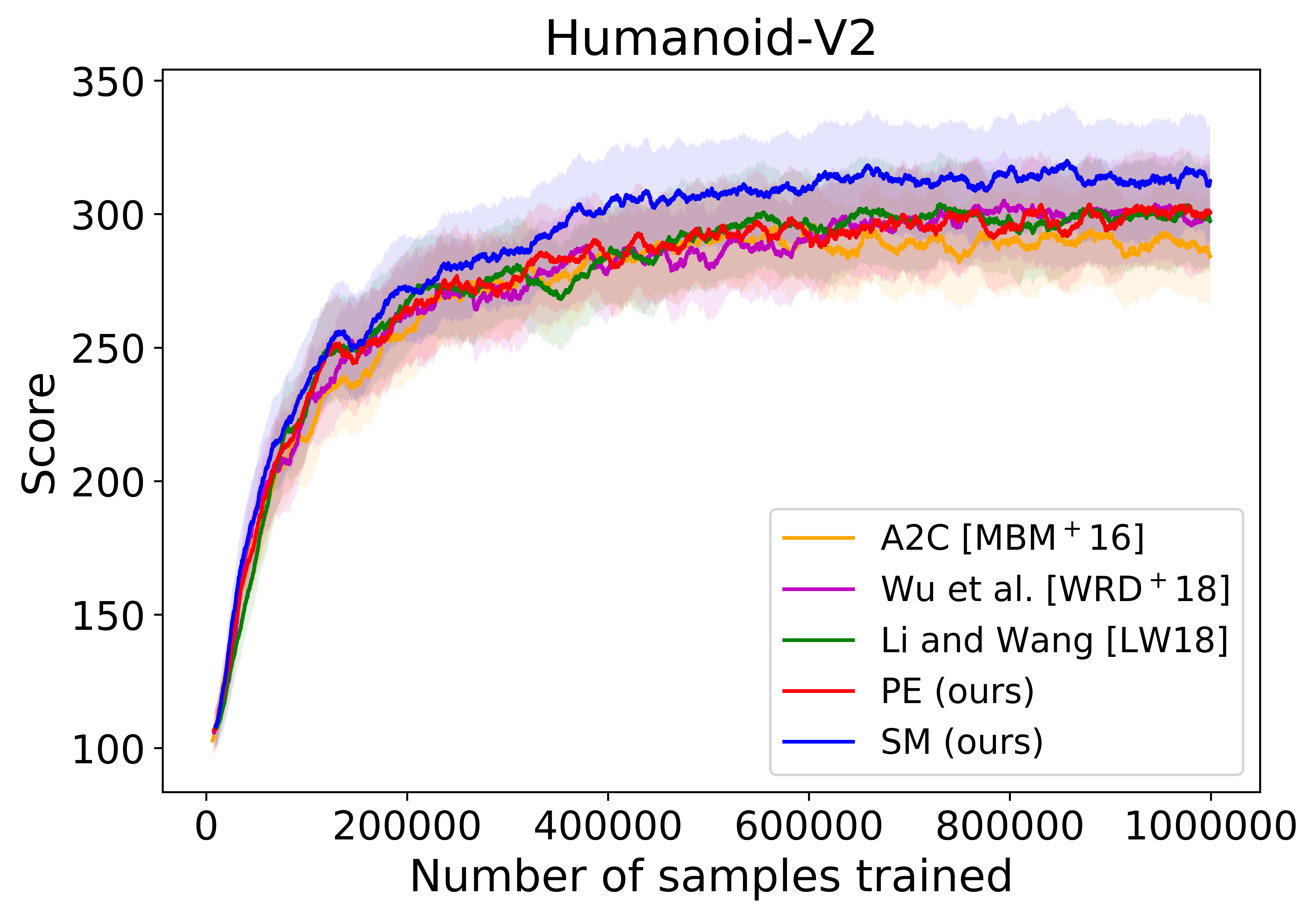}} 
{\includegraphics[width=0.24\textwidth]{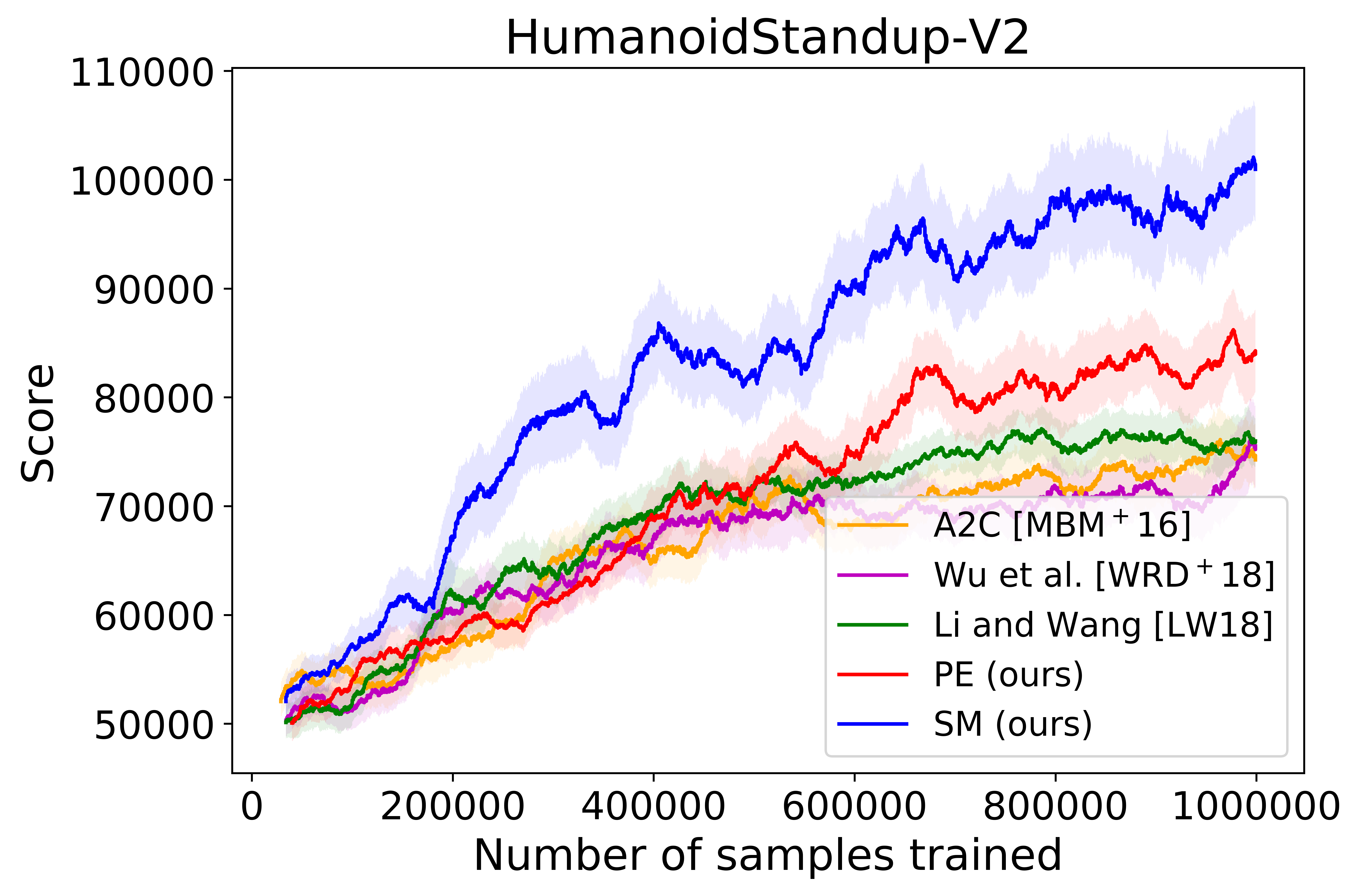}}
{\includegraphics[width=0.24\textwidth]{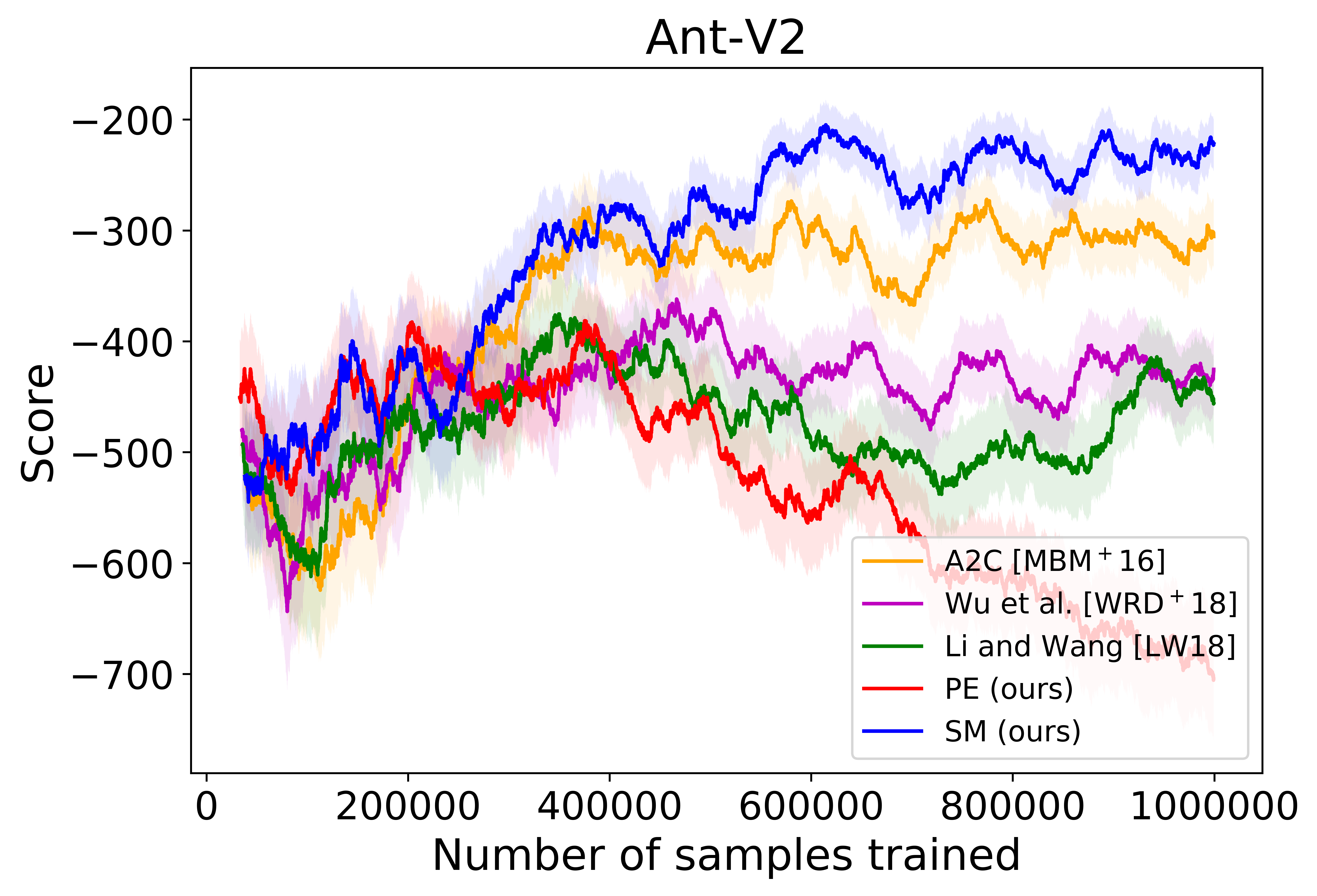}}
\caption{Empirical comparisons on MuJoCo high-dimensional control tasks. Each curve is averaged over 10 independent experiments.}
\label{fig:mujoco}
\end{figure*}

\medskip
\paragraph{Acknowledgments} We would like to thank Arnab Bhattacharyya and Guy Kindler for helpful discussions on our work and its connection to learning and testing juntas, Lap Chi Lau for telling us about the work of Blais et al.~\cite{blais2018tolerant}, and Jiajin Li for pointing out that variable partitioning for reinforcement learning in fact reduces the variance of its policy gradient estimator.

\bibliographystyle{alpha}
\bibliography{itcs2020-arxiv}

\newpage

\appendix

\section{Statistical claims}
\label{appendix:proof1}

\begin{claim}
\label{fact:pscaling}
Assume $t, \hat{t} \geq 0$.  If $t^p \leq \hat{t}^p \leq t^p + (\eps/2)^p$ then $t \leq \hat{t} \leq t + \eps$.
\end{claim}
\begin{proof}
The left-hand inequalities are immediate.  For the right-hand ones we start we consider two cases.  If $t \leq \eps/2$, then $\hat{t}^p \leq 2(\eps/2)^p \leq \eps \leq t + \eps$.  If $t > \eps/2$ then
\[ \hat{t} - t \leq \frac{\hat{t}^p - t^p}{t^{p-1}} \leq \frac{(\eps/2)^p}{(\eps/2)^{p-1}} \leq \eps. \hfill\tag*{\qedhere} \]
\end{proof}

\begin{namedtheorem}[Claim \ref{claim:estD}]
Assuming $\norm{F}_{\mathbb{R}, 2p} \leq 1$, the value $\norm{F}_{\mathbb{R}, p}^p$ can be estimated within $\eps^p$, and $\norm{F}_{\mathbb{R}, p}$ can be estimated within $\eps$, from $K^p{\log(1/\gamma)}/{\epsilon^{2p}}$ queries to $F$ in linear time with probability $1 - \gamma$ for some absolute constant $K$.
\end{namedtheorem}
\begin{proof}
By Chebyshev's inequality, $\E[\abs{F}^p]$ can be estimated within an additive error of $(\eps/2)^p$ by averaging $(2/\eps)^{2p}$ samples with probability $3/4$.  The error can be improved to $1 - \gamma$ by taking the median value of $\OO(\log 1/\gamma)$ runs.   The second bound follows from Claim~\ref{fact:pscaling}.
\end{proof}

\section{Details in the experiments}
\label{appendix:syn}

The exact reinforcement learning control algorithm we used is described below. The algorithm is based on proximal policy optimization \cite{schulman2017proximal} and generalized advantage estimator \cite{schulman2015high,degris2012off} in reinforcement learning.

\begin{algorithm}
  \caption{Policy optimization with variable partitions}
  \label{algo-pg}
  \begin{algorithmic}[1]
    \State{\bfseries Input:} Total number of samples $T$, batch size $B$, partition frequency $M_p$, number of value iterations $M_w$, initial policy parameter $\theta$, initial value and advantage parameters $w$ and $\mu$;
    \State{\bfseries Output:} Optimized policy $\pi_\theta$;
    \For{each iteration $j$ in $[T/B]$}
    \State Collect a batch of trajectory data $\{s_t^{(i)}, a_t^{(i)},r_t^{(i)}\}_{i=1}^B$;
    \For{$M_\theta$ iterations}
    \State Update $\theta$ by one gradient descent step using proximal policy gradient \blue{with the gradient estimator \eqref{gas}};
    \EndFor
    \For{$M_w$ iterations}
    \State Update $w$ and $\mu$ by minimizing $\|V^w(s_t)-R_t\|_2^2$ and $\|\hat{A}-A^\mu(s_t,a_t)\|_2^2$ in one step;
    \EndFor
    \State Estimate $\hat{A}(s_t,a_t)$ using $V^w(s_t)$ by generalized advantage estimator; \label{alg2:step-ppoend}
    \If{$j \equiv 0 \;(\bmod\; M_p)$} \label{alg2:step-partitionif}\tikzmark{top}
       \State Define random function $\xi(\varX)$ to be an estimation of $\EE[D_F(\varX, \varXbar)^2]$; \tikzmark{right}
       \State Run submodular minimization over $\varX$ on $\xi(\varX)$;
       \State Assign $\varX$ and $\bar{\varX}$ to $a_{(1)}$ and $a_{(-1)}$ in \eqref{gas}, respectively;
    \EndIf \label{alg2:step-partitionendif} \tikzmark{bottom}
    \EndFor
  \end{algorithmic}
  \AddNote{top}{bottom}{right}{Variable Partitioning}
\end{algorithm}

The differences between our algorithm and proximal policy gradient \cite{schulman2017proximal} have been highlighted: \blue{Line 6} uses the estimator with partitions on the control variables. \blue{Line $12$-$16$} find the near-optimal variable partition using submodular minimization, by Theorem \ref{thm:l2est}.

We use three neural networks as function approximations: a policy network $\pi_\theta$ and a value network $V^w$ as is in the baseline methods, and an advantage network $A^\mu$ solely used in the partition algorithm. The networks have the same architecture as is in the previous line of works \cite{mnih2016asynchronous,schulman2017proximal}. 

In our MuJoCo experiments, the tasks have been slightly modified (the physics simulator keeps intact). As the number of control variables of the original tasks is relatively low, we augment such dimensions by letting the agent controls two independent instances of the tasks at the same time. 
The scores and the reinforcement signals are then the additions of the scores of the two sub-tasks.
Correspondingly, we use $k=2$ in \cite{li2018policy} and our algorithms.
\end{document}